\documentclass{article}

\PassOptionsToPackage{square,numbers,sort&compress}{natbib}

 \usepackage[preprint]{neurips_2026}

\usepackage[utf8]{inputenc} 
\usepackage[T1]{fontenc}    
\usepackage{hyperref}       
\usepackage{url}            
\usepackage{booktabs}       
\usepackage{amsfonts}       
\usepackage{nicefrac}       
\usepackage{microtype}      
\usepackage{xcolor}         

\usepackage{amsmath}
\usepackage{amssymb}
\usepackage{mathtools}
\usepackage{amsthm}
\usepackage{algorithm}
\usepackage{algorithmic}
\usepackage{float}
\usepackage{multirow}
\usepackage{subcaption} 
\usepackage{subfiles}

\theoremstyle{plain}
\newtheorem{theorem}{Theorem}[section]
\newtheorem{proposition}[theorem]{Proposition}
\newtheorem{lemma}[theorem]{Lemma}

\theoremstyle{definition}
\newtheorem{definition}[theorem]{Definition}
\newtheorem{assumption}[theorem]{Assumption}
\newtheorem{remark}[theorem]{Remark}

\DeclareMathOperator*{\argmax}{arg\,max}

\DeclareMathOperator{\subopt}{SubOpt}

\DeclareMathOperator{\Gap}{Gap}

\DeclareMathOperator{\relint}{relint}
\DeclareMathOperator{\supp}{supp}
\DeclareMathOperator{\conv}{conv}

\title{Inverse Reinforcement Learning\\ without an Optimal Demonstrator:\\ A Feasible Reward Set Approach}

\author{%
  Kihyun Kim \\
  MIT LIDS \\
  \texttt{kihyun@mit.edu} \\
  \And
  Shripad Deshmukh \\
  University of Massachusetts, Amherst \\
  \texttt{svdeshmukh@cs.umass.edu} \\
  \And
  Nikos Vlassis \\
  Adobe Research \\
  \texttt{vlassis@adobe.com} \\
  \And
  Jiawei Zhang\thanks{Corresponding author.} \\
  University of Wisconsin-Madison \\
  \texttt{jzhang2924@wisc.edu} \\
}

\begin{document}

\maketitle

\begin{abstract}
Inverse reinforcement learning (IRL) typically assumes demonstrations from a single optimal demonstrator, but in many applications data come from multiple imperfect demonstrators with heterogeneous suboptimality levels. We study reward learning in this setting through a feasible-reward-set framework: for each demonstrator, we encode its declared suboptimality level as a linear constraint and intersect the resulting feasible sets across demonstrators. Our theoretical analysis shows that the joint feasible set shrinks monotonically as data are added, and we give an exact characterization of when a new demonstrator strictly tightens it. We further establish two recovery guarantees for the feasible reward set of the ground-truth optimal demonstrator: one bound depends on closeness to the optimal occupancy, while the other requires only sufficient coverage and no near-optimal demonstrator. On the practical side, we introduce strategies to address the inherent reward ambiguity in the obtained reward set and provide an offline algorithm with function approximation for high-dimensional environments. Experiments in tabular grid-world and large language model (LLM) fine-tuning settings are consistent with the theoretical predictions and demonstrate the effectiveness of the proposed framework over baselines.
\end{abstract}

\section{Introduction}\label{sec:1}

IRL aims to infer rewards that explain observed behavior from demonstrations, typically under the assumption that the demonstrator policy is optimal under the true reward~\cite{ng2000algorithms, adams2022survey}.
In many practical applications, however, demonstration data is not collected from a single optimal demonstrator. Instead, we often observe demonstrations from multiple imperfect demonstrators, each reflecting limited skill, resource constraints, or data corruption.
For example, in LLMs, supervision commonly comes from imperfect and heterogeneous human raters rather than from a single optimal expert~\cite{bukharin2024robust, park2024rlhf, zhang2025diverging}.
In robotics, demonstrations are often collected from operators with heterogeneous proficiency~\cite{chen2020joint,chen2021learning}, or in constrained domains~\cite{jayanthi2023droid}.
Such scenarios underscore the need for IRL algorithms that can recover optimal behavior from suboptimal demonstrations generated by multiple demonstrators with heterogeneous optimality levels.

Conventional IRL algorithms, such as maximum entropy (MaxEnt) IRL~\citep{ziebart2008maximum} and Bayesian IRL~\citep{ramachandran2007bayesian}, typically assume a probabilistic behavioral model and recover a single reward consistent with that model.
Building on these algorithms, \citep{beliaev2022imitation,beliaev2025inverse} account for varying demonstrator quality by estimating demonstrator expertise or demonstrator-specific noise parameters.
However, this approach leaves a fundamental question unresolved: can rewards under which the ground-truth demonstrator's policy is optimal be provably recovered from suboptimal demonstrations alone?
The difficulty is that maximum likelihood estimation (MLE) tends to assign higher reward to frequently visited state-action pairs, even when they do not reflect optimal behavior.
Ranked-demonstration methods address this question more directly, giving a sufficient condition for extrapolation beyond suboptimal demonstrations~\cite{brown2019extrapolating, brown2020better}.
However, their theory guarantees performance improvement rather than recovery of the reward set under which the ground-truth demonstrator's policy is optimal.
Moreover, their ranking-based algorithm does not apply without strict preferences, such as when demonstrators have the same optimality level but specialize in different tasks.

To address this question, we adopt the \emph{feasible reward set} approach in this paper. Rather than selecting a single reward under an assumed model, this approach characterizes the set of rewards consistent with demonstrated behavior through linear constraints~\citep{metelli2021provably, metelli2023towards}.
Specifically, for each demonstrator, we construct the feasible set of rewards under which its behavior is at most $\epsilon$-suboptimal, where $\epsilon$ is a given upper bound on that demonstrator's suboptimality. We then intersect these sets across demonstrators to obtain a joint feasible set.
This approach provides a model-free way to encode the suboptimality levels of the demonstrators, rather than representing them indirectly through a probabilistic model assumption.

This offers several advantages over previous work.
On the theory side, we show that the joint feasible set shrinks monotonically as demonstrations are added, and give an exact characterization of when a new demonstrator strictly tightens it.
We further provide two recovery guarantees for the feasible reward set of the ground-truth optimal demonstrator: one based on closeness to the optimal occupancy, and another based on sufficient coverage without requiring a near-optimal demonstrator.
We complement these results with instances where existing MaxEnt and ranking-based IRL baselines retain constant mismatch, while our method exactly recovers the feasible reward set.
On the practical side, we introduce practical strategies to address reward ambiguity issues, and extend the proposed framework to high-dimensional environments with function approximation.
Taken together, these results yield a principled, model-free framework for IRL from multiple suboptimal demonstrations, providing an alternative to model-dependent likelihood-based methods.
Experiments in both tabular grid-world and high-dimensional LLM fine-tuning settings support the theoretical predictions and demonstrate the effectiveness of the proposed framework over baselines.

\paragraph{Related work.}
Our work builds on optimization-based views of IRL, including the original constraint formulation~\citep{ng2000algorithms} and recent feasible-reward-set analyses~\citep{metelli2021provably,metelli2023towards,kim2024unified}.
\citet{poiani2024sub} show that suboptimal-demonstrator data shrink the feasible reward set, but assume access to an optimal demonstrator, whereas our focus is on reward-set recovery without one.
\citet{lazzati2025generalizing} study principled representative selection from the feasible reward set via closed-form reward centroids; we adopt a related idea to address reward ambiguity issues.
Learning from imperfect, failed, or negatively labeled behavior has also been studied in IRL and imitation learning (IL): failed demonstrations can reduce reward ambiguity~\citep{shiarlis2016inverse}, ranked-demonstration methods learn from ordered imperfect trajectories~\citep{brown2019extrapolating,brown2020better}, margin-based methods use separation constraints~\citep{ratliff2006maximum,ratliff2009learning}, and other work learns policies directly from imperfect demonstrations~\citep{gao2018reinforcement,zhang2021confidence}.
In contrast to these works, our work studies reward-set identifiability under declared suboptimality bounds, without relying on probabilistic model assumptions, ordered demonstrations, or access to an optimal demonstrator. A more detailed discussion is provided in Appendix~\ref{app:related}.

\paragraph{Organization of the paper.}
Section~\ref{sec:2} introduces notation and background on the LP view of IRL.
Section~\ref{sec:3:1} presents the feasible reward set formulation for multiple suboptimal demonstrations, establishes monotonic shrinkage and recovery guarantees, and compares with existing baselines.
Section~\ref{sec:3:2} introduces strategies for addressing the reward ambiguity issue in IRL, and Section~\ref{sec:3:3} proposes an offline algorithm with function approximation for high-dimensional environments.
Finally, Section~\ref{sec:4} provides empirical results on both a grid-world environment and LLM fine-tuning.
\section{Preliminaries}\label{sec:2}
Let $[K]$ denote the set $\{1, 2, \ldots, K\}$ and let $\Delta(S)$ denote the probability simplex over $S$.
$\Vert \cdot \Vert_p$ denotes the $\ell_p$ norm.
$\mathbf{1}\{e\}\in\{0,1\}$ denotes the indicator function, which equals $1$ if event $e$ occurs and $0$ otherwise.
We use $\mathbf{1}$ to denote the all-ones vector of appropriate dimensions.
$\supp(\cdot)$ denotes the support of a function (or distribution) and $\conv(\cdot)$ denotes the convex hull.
For ease of notation, we use both function and vector notation for rewards and occupancy measures.

\subsection{Discounted Markov decision processes (MDPs) and occupancy measures}
We consider an infinite-horizon discounted MDP $\mathcal{M}_r := (S,A,P,r,\mu_0,\gamma)$, where $S$ and $A$ are finite state and action spaces, $P(\cdot|s,a)\in\Delta(S)$ is the transition kernel, $r:S\times A\mapsto [0,1]$ is the reward function, $\mu_0\in\Delta(S)$ is the initial-state distribution, and $\gamma\in(0,1)$ is the discount factor.
When the reward is not specified, we write $\mathcal M := (S,A,P,\mu_0,\gamma)$ for the corresponding reward-free MDP. For a policy $\pi(\cdot| s)\in\Delta(A)$, the value function $v^\pi_r(s)$ is defined as the expected total discounted reward received when initiating from state $s$ and following $\pi$, such that
\begin{equation}
v_r^\pi(s):=\mathbb{E}^{\pi} \left[\sum_{h=0}^{\infty}\gamma^h r(s_h,a_h)\mid s_0=s \right]  .  
\end{equation}
Let $v_r^\star$ denote the optimal value function that maximizes the value function, which satisfies the Bellman equations.
Then, it is well known that $v_r^\star$ is the optimal solution to the following linear program~\citep{puterman1994markov}:
\begin{equation}\label{eq:primal}
\min_{v\in\mathbb{R}^{|S|}} (1-\gamma)\mu_0^\top v
\quad\text{s.t.}\quad
M^\top v \ge r,
\end{equation}
where the dynamics operator $M\in\mathbb{R}^{|S|\times(|S||A|)}$ is defined by $
M(s',(s,a)) := \mathbf{1}\{s'=s\}-\gamma P(s'| s,a)$.
The dual program can be expressed as
\begin{equation}\label{eq:dual}
\max_{d\in\mathbb{R}^{|S||A|}} r^\top d
\quad\text{s.t.}\quad
Md=(1-\gamma)\mu_0,\quad d\ge 0.
\end{equation}
The dual variable $d\in\mathbb{R}^{|S||A|}$ is a \emph{discounted occupancy measure}, i.e., the normalized discounted frequency with which a policy visits each state--action pair.
For a policy $\pi$ and initial distribution $\mu_0$, $d^\pi$ is defined by
\begin{equation}
d^\pi(s,a)
:= (1-\gamma)\,\mathbb{E}^{\pi}\!\left[\sum_{h=0}^{\infty}\gamma^h\,\mathbf{1}\{(s_h,a_h)=(s,a)\}\,\Big|\, s_0\sim\mu_0\right]
\end{equation}
which is feasible for the dual LP~\eqref{eq:dual}.

Throughout the paper, we fix the reward-free MDP $\mathcal M$ and denote the admissible occupancy polytope under $\mathcal M$ by
\begin{equation}
\Phi := \{d\in\mathbb{R}_+^{|S||A|}\mid Md=(1-\gamma)\mu_0\}.
\end{equation}
For a reward $r$, define the suboptimality gap of an occupancy measure $d\in\Phi$ as
\begin{equation}
\subopt(r,d):=\max_{\tilde{d}\in\Phi}\; r^\top \tilde{d} - r^\top d.
\end{equation}
Then, $\subopt(r,d)=0$ if and only if $d$ maximizes the reward over $\Phi$. We denote the set of optimal occupancy measures under reward $r$ by $\Phi^\star(r) \subseteq \Phi$:
\begin{equation}\label{eq:optimal_occupancy_set}
\Phi^\star(r):=\{d\in\Phi\mid \subopt(r,d)=0\}
= \argmax_{d\in\Phi}\; r^\top d.
\end{equation}

\subsection{IRL from a single optimal demonstrator}\label{sec:2:3}
\paragraph{Feasible reward set.}
Given an optimal demonstrator policy $\pi_e$, let $d_e:=d^{\pi_e}\in\Phi$ denote its occupancy measure, where the subscript $e$ denotes expert.
Since multiple rewards can make the demonstrator occupancy measure optimal, we consider the set of rewards for which $d_e$ is optimal, referred to as the feasible reward set in the IRL literature \cite{metelli2021provably, metelli2023towards}:
\begin{equation}\label{eq:R_single}
\mathcal{R}(d_e):=\{r\in \Delta(S \times A)\mid d_e\in \Phi^\star(r)\}.
\end{equation}
The simplex constraint $r \in \Delta(S \times A)$ is introduced to remove scaling ambiguity in the suboptimality measure, and all feasible reward sets in the theory below use this normalization.
For other reward scales, the suboptimality levels must be rescaled by the corresponding constant.
With this normalization, the feasible reward set has a simple geometric interpretation: it is the intersection of the normal cone to $\Phi$ at $d_e$ and the simplex $\Delta(S \times A)$, i.e., $\mathcal{R}(d_e) =  N (d_e;\Phi) \cap \Delta(S \times A)$, where $N (d_e;\Phi) := \{ r \mid r^\top(d_e - d) \geq 0 \quad \forall d \in \Phi \}$.

\paragraph{LP characterization.}
By strong duality, $d_e$ is optimal for the dual LP~\eqref{eq:dual} under reward $r$ if and only if there exists a primal variable $v$ such that Karush--Kuhn--Tucker (KKT) conditions hold:
\begin{equation}\label{eq:R_kkt}
\begin{split}
(1-\gamma)\mu_0^\top v - r^\top d_e = 0 \quad & \text{(Zero duality gap)},\\
M^\top v \ge r \quad & \text{(Primal feasibility)}.    
\end{split}
\end{equation}
Note that the dual feasibility condition is unnecessary since it is automatically satisfied by the assumption $d_e \in \Phi$.
Then, the feasible reward set can be expressed as the following polytope~\cite{kim2024unified}
\begin{equation}\label{eq:single_lp_formulation}
\mathcal{R}(d_e) = \{\; r \in \Delta(S \times A) \mid \exists v\in \mathbb{R}^{|S|} \quad \text{s.t.} \quad (1-\gamma)\mu_0^\top v - r^\top d_e = 0, \; M^\top v \geq r \;\}. 
\end{equation}

\section{IRL from multiple suboptimal demonstrators}\label{sec:3}
\subsection{Proposed IRL framework with theoretical guarantees}\label{sec:3:1}
\paragraph{Feasible reward set.}
We first extend the feasible reward set of a single optimal demonstrator to the multiple suboptimal demonstrations setting.
Consider $K$ demonstrators with policies $\{\pi_e^k\}_{k=1}^K$ and corresponding discounted occupancy measures $d_e^k := d^{\pi_e^k}$.
Suppose that each demonstrator policy is known to be at maximum $\epsilon^k$ suboptimal under true rewards, where $\epsilon^k \in [0, 1]$.
Let $\mathcal{D} := \{(d_e^k, \epsilon^k)\}_{k=1}^K$ denote the set of occupancy-suboptimality pairs.

Unless we have $\epsilon^k = 0$ for some $k$, we do not have access to an optimal demonstrator.
Therefore, we first recover the feasible reward set that makes $d_e^k$ $\epsilon^k$-suboptimal, and find the intersection over all $k$.
Analogously to the optimal occupancy set $\Phi^\star(r)$~\eqref{eq:optimal_occupancy_set} in IRL from a single optimal demonstrator, we define the $\epsilon$-suboptimal occupancy set for any reward $r$ as
\begin{equation}
\Phi^\star(r; \epsilon) := \{d\in\Phi \mid \subopt(r,d)\le \epsilon\}.
\end{equation}
We then define the feasible reward set for which $d_e^k$ is at most $\epsilon^k$ suboptimal, and take the intersection over all demonstrators:
\begin{definition}[Feasible reward set for multiple suboptimal demonstrations]\label{def:feasible_reward_set}
\begin{equation}
    \mathcal{R} (\mathcal{D}) := \bigcap_{k=1}^K \mathcal{R}(d_e^k; \epsilon^k), \quad \text{where} \quad \mathcal{R}(d_e^k; \epsilon^k) := \{r \in \Delta(S \times A) \mid d_e^k \in \Phi^\star(r; \epsilon^k) \}.
\end{equation}
\end{definition}

\begin{proposition}\label{prop:lp_formulation}
The feasible reward set $\mathcal{R} (\mathcal{D})$ can be expressed as:
\begin{equation}\label{eq:lp_formulation}
\mathcal{R} (\mathcal{D}) := \{r \in \Delta(S \times A) \mid \exists v\in \mathbb{R}^{|S|} \; \text{s.t.} \; (1-\gamma)\mu_0^\top v - r^\top d_e^k \leq \epsilon^k \quad \forall k \in [K], \; M^\top v \geq r\}.
\end{equation}
\end{proposition}
The proof is provided in Appendix~\ref{app:prop:lp_formulation}.
The 2D illustration of the feasible reward sets for a single optimal demonstrator and multiple suboptimal demonstrators is shown in Figure~\ref{fig:feasible_reward_set}(a) and (b), respectively.

\begin{figure}
    \centering
    \begin{subfigure}[b]{0.33\textwidth}
    \centering
        \includegraphics[width=0.85\textwidth]{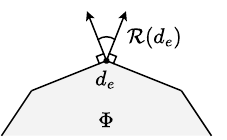}
        \caption{}
    \end{subfigure}
    \begin{subfigure}[b]{0.32\textwidth}
    \centering
        \includegraphics[width=\textwidth]{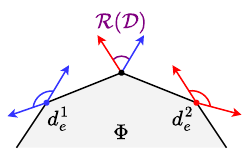}
        \caption{}
    \end{subfigure}
    \begin{subfigure}[b]{0.32\textwidth}
        \centering
            \includegraphics[width=\textwidth]{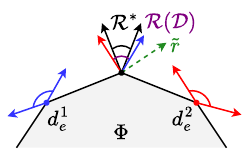}
            \caption{}
    \end{subfigure}
    \caption{2D illustration of (a) the feasible reward set for a single optimal demonstrator (black cone), (b) the feasible reward set with two suboptimal demonstrators (purple cone, intersection of blue and red cones), and (c) Theorem~\ref{thm:bound2}, where $\tilde r$ (green arrow) is ruled out by demonstrator $d_e^1$ (blue cone). }
    \label{fig:feasible_reward_set}
\end{figure}

\paragraph{Monotonic shrinkage property.}
First, we investigate the monotonic shrinkage property of the feasible reward set as new demonstrations are added.
Specifically, we show that the feasible reward set becomes strictly tighter when a new demonstrator visits previously unseen state-action pairs and its declared suboptimality level is sufficiently small relative to the newly covered mass.
This follows from the fact that our formulation is intersection-based: each additional demonstration contributes another linear compatibility constraint.
We first define the ground-truth optimal policy and the target reward set we aim to recover.
\begin{definition}[Ground-truth solution]\label{def:ground_truth}
Let $\pi_e^\star$ denote the ground-truth optimal policy and $d_e^\star := d^{\pi_e^\star} \in \Phi$ denote its occupancy measure.
Let $\mathcal{R}^\star \subseteq \mathcal{R}(d_e^\star)$ denote the target reward set that makes $\pi_e^\star$ an optimal policy.
\end{definition}  
Note that neither $d_e^\star$ nor $\mathcal{R}^\star$ can be directly estimated because demonstrations from the optimal demonstrator $\pi_e^\star$ are unavailable.
Instead, we observe the dataset $\mathcal D:=\{(d_e^k,\epsilon^k)\}_{k=1}^K$ from $K$ suboptimal demonstrators under the following assumption: the declared suboptimality levels do not exclude any reward in the target reward set.
\begin{assumption}[Dataset consistency]\label{ass:consistency}
The dataset $\mathcal D$ is consistent with $\mathcal R^\star$ in the sense that it does not rule out any reward in $\mathcal R^\star$: $\subopt(r, d_e^k) \leq \epsilon^k \quad \forall r \in \mathcal{R}^\star,\; k \in [K]$.
Equivalently, every reward in $\mathcal R^\star$ remains feasible under the dataset, i.e., $\mathcal R^\star \subseteq \mathcal R(\mathcal D)$.
\end{assumption}

The following theorem gives a coverage-based sufficient condition under which adding a new demonstrator strictly shrinks the feasible reward set while preserving the target reward set.
\begin{theorem}\label{thm:monotonicity}
Suppose Assumption~\ref{ass:consistency} holds for the dataset $\mathcal{D}$ and the augmented dataset $\mathcal{D}^+ := \mathcal{D}\cup\{(d_e',\epsilon')\}$.
Let $\mathcal Z_{\mathcal D}:=\bigcup_{k=1}^K \supp(d_e^k)$ be the set of state-action pairs visited in $\mathcal D$, and let $\alpha(d_e', \mathcal D) := \sum_{(s,a)\notin \mathcal Z_{\mathcal D}} d_e'(s,a)$ be the mass that the new demonstrator places outside the previously visited set.
If
\begin{equation}
    \epsilon' < \frac{\alpha(d_e', \mathcal D)}{|\mathcal Z_{\mathcal D}|},
\end{equation}
then $\mathcal{R}^\star \subseteq \mathcal{R}(\mathcal{D}^+) \subsetneq \mathcal{R}(\mathcal{D})$.
\end{theorem}
The proof is provided in Appendix~\ref{app:thm:monotonicity}.
We make a few remarks about this theorem.
First, the theorem provides only a sufficient condition on the new data for strictly shrinking the set $\mathcal{R}(\mathcal{D})$, and the necessary and sufficient condition is given in Lemma~\ref{lem:strict_shrinkage_characterization}.
Second, the denominator $|\mathcal Z_{\mathcal D}|$ may appear large because it scales with the number of visited state-action pairs, but this dependence is a consequence of the simplex normalization in our feasible reward set definition: under this normalization, the average reward over the full state-action space is only $1/(|S||A|)$.
If a different constant reward scale is used, the condition rescales accordingly.

\paragraph{Recovery of $\mathcal{R}(d_e^\star)$.}
Next, we bound the mismatch between $\mathcal{R}(\mathcal D)$ and the feasible reward set of the ground-truth optimal demonstrator, $\mathcal{R}(d_e^\star)$.
To quantify this mismatch, we use the worst-case suboptimality of $d_e^\star$ over rewards in $\mathcal{R}(\mathcal D)$:
\begin{equation}
\Gap(\mathcal{R}(\mathcal D), \mathcal{R}(d_e^\star)) := \max_{r \in \mathcal{R}(\mathcal D)} \subopt(r,d_e^\star).
\end{equation}
If $\Gap(\mathcal{R}(\mathcal D), \mathcal{R}(d_e^\star)) \to 0$, then $\mathcal{R}(\mathcal D)$ converges to $\mathcal{R}(d_e^\star)$ in the one-sided Hausdorff distance induced by the $\ell_\infty$ norm (see Appendix~\ref{app:hausdorff_distance} for the proof).
If $\Gap(\mathcal{R}(\mathcal D), \mathcal{R}(d_e^\star)) = 0$, then $\mathcal{R}(\mathcal D) \subseteq \mathcal{R}(d_e^\star)$, so any selected reward in $\mathcal{R}(\mathcal D)$ exposes an optimal face of $\Phi$ containing $d_e^\star$; under the relative-interior normal-cone condition, this exposed face is the minimal face containing $d_e^\star$ (see Appendix~\ref{app:optimal_face_recovery} for details).

The following two theorems present the upper bounds under two different cases.
The first case provides the bound which depends on the closeness of the convex hull of the suboptimal demonstrators' occupancies to the optimal occupancy.
\begin{theorem}\label{thm:bound1}
Suppose that Assumption~\ref{ass:consistency} holds.
Then, the following inequality holds:
 \begin{equation}
    \Gap(\mathcal{R}(\mathcal D), \mathcal{R}(d_e^\star)) \leq \min_{(d, \epsilon) \in \conv(\mathcal{D})}\left\{\Vert d - d_e^\star \Vert_1  + \epsilon \right\},
\end{equation}
\end{theorem}
The proof is provided in Appendix~\ref{app:thm:bound1}.
The second case is less trivial than the previous one, as it does not rely on the closeness to the optimal occupancy measure, but only requires the sufficient coverage of the suboptimal demonstrations.
\begin{theorem}\label{thm:bound2}
Suppose that Assumption~\ref{ass:consistency} holds.
In addition, suppose that for every $\tilde r$ such that $\subopt(\tilde r,d_e^\star)>\delta$, there exist $r' \in \mathcal{R}(d_e^\star)$ and $(d_e^{k}, \epsilon^{k}) \in \mathcal{D}$ such that
\begin{equation}
(\tilde r-r')^\top(d_e^\star - d_e^{k}) > 0, \;\; \text{and} \;\;
\subopt(r', d_e^{k}) \in [\epsilon^{k} - \delta, \epsilon^{k}].
\end{equation}
Then, $\Gap(\mathcal{R}(\mathcal D), \mathcal{R}(d_e^\star)) \leq \delta$.
\end{theorem}
The proof is provided in Appendix~\ref{app:thm:bound2}.
The theorem states how demonstrator constraints rule out rewards that have large mismatch with $\mathcal R(d_e^\star)$.
For any reward $\tilde r$ that makes $d_e^\star$ more than $\delta$-suboptimal, there must be a reference reward $r' \in \mathcal R(d_e^\star)$ and a demonstrator $(d_e^k,\epsilon^k)\in\mathcal D$ that exposes this mismatch.
At $r'$, demonstrator $k$ is already close to its allowed suboptimality level, so its feasibility constraint has little slack.
Moving from $r'$ toward $\tilde r$ then increases the return of the optimal demonstrator $d_e^\star$ more than that of demonstrator $d_e^k$, causing $\tilde r$ to violate this demonstrator constraint.
This mechanism is illustrated in Figure~\ref{fig:feasible_reward_set}(c).

\begin{remark}[Finite demonstrators suffice]
While Theorem~\ref{thm:bound2} states the pointwise assumption for infinitely many rewards, it can be satisfied by finitely many demonstrators.
Indeed, if the local condition holds on the closed set of rewards with $\subopt(\tilde r,d_e^\star)\ge\delta$, compactness gives a finite subcover and hence a finite dataset.
A formal statement and proof is given in Appendix~\ref{app:finite_demonstrators}.
\end{remark}
 
\paragraph{Comparison with baselines.}
Next, we compare our approach with two standard baselines for multiple suboptimal demonstrations: a MaxEnt estimator with demonstrator-specific inverse temperatures $\beta_k$ fixed from the declared suboptimality levels~\citep{beliaev2022imitation,beliaev2025inverse}, and a ranking-based estimator based on Bradley--Terry--Luce (BTL) pairwise likelihoods~\citep{bradley1952rank,luce1959individual}, as in T-REX/D-REX-style methods~\citep{brown2019extrapolating,brown2020better}.
Both baselines optimize a single global objective over all demonstrations, unlike our intersection-based feasible-set formulation, and therefore need not track the shrinking feasible reward set.
The next proposition shows that each can retain a constant mismatch with $\mathcal{R}(d_e^\star)$ even when the bound in Theorem~\ref{thm:bound1} is arbitrarily small; formal definitions and proofs are in Appendices~\ref{app:maxent} and~\ref{app:ranking}.

\begin{proposition}\label{prop:baseline-failure}
For each of the MaxEnt and ranking-based estimators described in Appendix~\ref{app:maxent} and Appendix~\ref{app:ranking}, there exists a constant $C>0$ such that, for every $\eta>0$, one can construct a finite dataset $\mathcal D$ for which the upper bound in Theorem~\ref{thm:bound1} is at most $\eta$, but every estimator $\hat r$ satisfies $\subopt(\hat r,d_e^\star)\ge C$.
\end{proposition}

\subsection{Addressing reward ambiguity issue in IRL}\label{sec:3:2}
\paragraph{Reward ambiguity issue.}
The reward set $\mathcal{R}(\mathcal{D})$ suffers from the reward ambiguity issue as in IRL from a single optimal demonstrator~\citep{ng2000algorithms,metelli2021provably,metelli2023towards,kim2024unified}.
The main issue is the existence of \emph{degenerate or sparse rewards} in $\mathcal{R}(\mathcal{D})$, such as $r = \tfrac{1}{|S||A|}\mathbf{1}$, that cannot separate the demonstrators' behaviors from other suboptimal behaviors.
An LP formulation that only imposes constraints in \eqref{eq:lp_formulation} does not rule out such degenerate solutions, and a solver may return one depending on the objective or tie-breaking rule.
We consider two complementary ways to reduce this ambiguity: imposing additional constraints and setting the objective function.

\paragraph{Imposing additional constraints.}
We consider two ways of imposing additional constraints beyond the constraints in \eqref{eq:lp_formulation} to reduce the feasible reward set.
First, a coverage-based penalty can discourage rewards that make unobserved state--action pairs appear optimal by enforcing positive Bellman inequality slack on pairs with zero aggregate demonstrator occupancy.
Geometrically, this is closely related to selecting rewards from the interior of the feasible set~\citep{lazzati2025generalizing}.
Second, pairwise performance-gap information, when available, can be encoded as additional linear constraints.
These constraints further shrink $\mathcal R(\mathcal D)$ and rule out rewards that cannot separate better demonstrators from worse ones.
Formal definitions and their advantages are discussed in detail in Appendices~\ref{app:coverage_penalty} and~\ref{app:pairwise_performance_gap}.

\paragraph{Reward selection within the feasible set.}
Even after these constraints are imposed, the feasible set may contain many rewards.
We select a representative reward using a baseline policy known, or assumed, to be substantially worse than the demonstrators, by maximizing the return gap between the average demonstrator and baseline occupancies:
\begin{equation}\label{eq:objective}
\hat r \in \argmax_{r\in \mathcal{R}(\mathcal{D})}  \; r^\top(\bar d_e - d_\mathrm{base}).
\end{equation}
Here, $\bar d_e:=\frac{1}{K}\sum_{k=1}^K d_e^k$ is the average demonstrator occupancy and $d_\mathrm{base}$ is the occupancy of the baseline policy $\pi_\mathrm{base}$, which is assumed to be inexpensive to estimate (when no natural baseline is available, one can use the uniform policy).
Although the maximizer need not be unique, this objective excludes rewards that assign the same value to the demonstrators and the baseline.
This strategy is particularly well-suited to LLM fine-tuning, where the pretrained base model provides a baseline and demonstrators improve over it only on the regions where they are competent.

\subsection{Offline algorithm with function approximation} \label{sec:3:3}
\paragraph{Datasets and function approximators.}
Next, we describe an offline function-approximation implementation of the tabular objective \eqref{eq:objective} for high-dimensional MDPs.
The goal is to preserve the structure of the LP: maximize the empirical return gap between demonstrators and a baseline, while enforcing the demonstrator-suboptimality and Bellman-feasibility constraints in \eqref{eq:lp_formulation} through penalties.
For each demonstrator $k$, we observe an offline dataset of $N_k$ trajectories of horizon $H$, sampled i.i.d. from the initial state distribution $\mu_0$, the unknown transition $P$, and $\pi_e^k$:
\begin{equation}
\mathcal{D}^k = \left\{\tau^{k,n}=(s^{k,n}_0,a^{k,n}_0,\ldots,s^{k,n}_{H})\right\}_{n=1}^{N_k}.
\end{equation}
Then, the empirical estimate of $d_e^k$ is computed from $\mathcal{D}^k$ by the standard truncated discounted occupancy estimate:
\begin{equation}
\hat{d}_e^k(s,a) := (1-\gamma)\frac{1}{N_k}\sum_{n=1}^{N_k}\sum_{h=0}^{H-1}\gamma^h \mathbf{1}\{(s_h^{k,n},a_h^{k,n})=(s,a)\} \quad \forall (s,a)\in S\times A.
\end{equation}

We parameterize the reward function and the value function as $r_\theta: S\times A\mapsto [0,1]$ and $v_\phi: S \mapsto\mathbb{R}$, respectively.
Given an offline dataset $\mathcal{D}=\{(\mathcal{D}^k, \epsilon^k)\}_{k=1}^K$, we train $(\theta,\phi)$ by optimizing the following empirical penalized objective.

\paragraph{Loss function.}
The first term below is the empirical version of the reward-selection objective \eqref{eq:objective}; the remaining terms penalize violations of the two LP constraints in \eqref{eq:lp_formulation}.
\begin{align}
    L(r_\theta,v_\phi) := &-r_\theta^\top \left( \frac{1}{K} \sum_{k=1}^K \hat{d}_e^k - \hat{d}_{\mathrm{base}} \right) + \lambda_\mathrm{sub} \frac{1}{K} \sum_{k=1}^K \; \rho_{\leq} \Big((1-\gamma) \mu_0^\top v_\phi- r_\theta^\top \hat{d}_e^k-\epsilon^k\Big) \label{eq:loss_gap}\\
    &+ \lambda_\mathrm{bell} \mathbb{E}_{(s,a)\sim \bar{\mathcal{D}}}\big[\rho_{\leq} \big(r_\theta(s,a) + \gamma \mathbb E_{s'\sim \bar{\mathcal D}(\cdot\mid s,a)}[v_\phi(s')] - v_\phi(s)\big)\big] \label{eq:loss_feas},
\end{align}
where $\lambda_\mathrm{sub}, \lambda_\mathrm{bell} > 0$ weight the demonstrator suboptimality penalty and the Bellman feasibility penalty, respectively, and $\rho_{\leq}$ denotes a convex non-decreasing penalty function, such as squared hinge $\rho_{\leq}(x)=\max\{0,x\}^2$.
Here $(s,a)\sim\bar{\mathcal D}$ denotes the empirical state-action distribution of the full offline dataset, and $\bar{\mathcal D}(\cdot\mid s,a)$ denotes its empirical next-state distribution.
This Bellman term is therefore a support-limited relaxation of the pointwise constraint $M^\top v\ge r$ in \eqref{eq:lp_formulation}.
$(\theta,\phi)$ can be updated by stochastic gradient methods on $L(r_\theta,v_\phi)$, enforcing bounded rewards by parameterization (e.g., sigmoid layer).

\begin{remark}
The proposed framework assumes the availability of upper bounds on each demonstrator's suboptimality level, but in practice these quantities may be unknown.
If these quantities are unknown, we treat them as hyperparameters controlling constraint strictness: larger values make more rewards feasible, while smaller values require demonstrations to be closer to optimal under the learned reward.
In practice, one can start with mild values, and gradually decrease them while monitoring validation performance and feasibility penalties.
\end{remark}

\paragraph{Complexity and convergence.}
In the tabular setting, \eqref{eq:objective} is an LP with $|S||A| + |S|$ variables, $K$ demonstrator-gap constraints, Bellman inequalities over $S\times A$, and the simplex constraints, so generic LP methods run in polynomial time in these dimensions.
Under function approximation, the cost scales with the number of gradient updates, the number of sampled transitions and demonstrators included in each batch, and the forward/backward cost of the reward and value networks.
When an explicit reward is not needed, this cost can be reduced by directly parameterizing the advantage function instead of separately learning $r_\theta$ and $v_\phi$; the resulting advantage-parameterized objective is given in Appendix~\ref{app:direct_policy}.
While convergence of specific optimization algorithms is beyond the scope of this paper, the empirical objective can be viewed as a standard sample-average approximation of its population analogue~\citep{shapiro2021lectures}. Sample-complexity bounds for estimating $\hat d_e^k$ are available in~\citep{kim2024unified}.

\paragraph{Limitations.}
We close this section by noting two limitations.
First, this paper focuses on the theoretical advantages of the proposed framework rather than on practical considerations for deployment.
The proposed function-approximation algorithm is a direct penalized extension of the tabular LP formulation, leaving implementation efficiency, optimizer choice, and online extensions for future work.
Second, as discussed above, the framework assumes access to upper bounds on demonstrator suboptimality, which may not be available in practice.
While we provide a practical tuning heuristic, more principled methods for selecting these quantities remain an important direction for future work.
\section{Numerical experiments}\label{sec:4}
\begin{figure}
    \centering
    \begin{subfigure}[b]{0.48\textwidth}
    \centering
    \includegraphics[width=\textwidth]{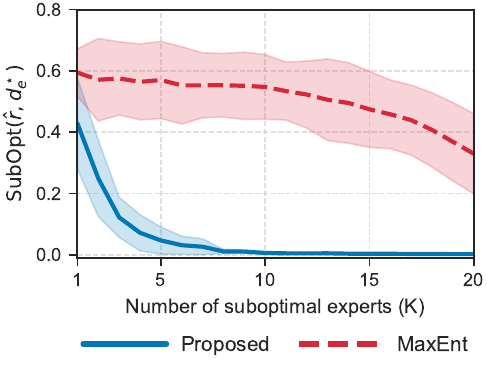}
    \end{subfigure}
    \hspace{0.02\textwidth}
    \begin{subfigure}[b]{0.48\textwidth}
    \centering
    \includegraphics[width=\textwidth]{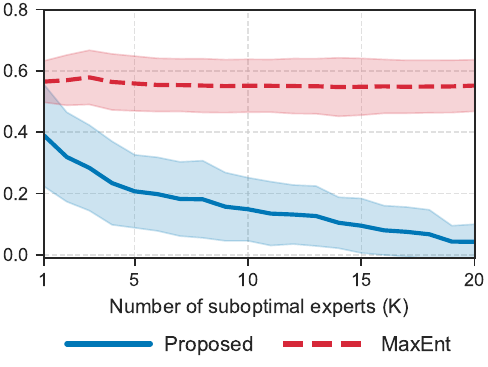}
    \end{subfigure}
    \caption{Suboptimality of the recovered reward for the proposed method and the MaxEnt baseline. Curves denote means and shaded bands denote one standard deviation over 20 random maps. Left: Case 1 (increasing quality and coverage). Right: Case 2 (fixed quality with increasing coverage).}
    \label{fig:grid_world}
    \end{figure}

\subsection{Grid-world with detour portals}
\paragraph{MDP setting.}
We use a tabular deterministic $20\times 20$ Grid-world with terminal state $(20,20)$, uniform initial distribution, and $20$ detour portals $\{(i_p,o_p)\}_{p=1}^{20}$ with distinct entrances.
The action space is $A=\{\textsf{U},\textsf{D},\textsf{L},\textsf{R},\textsf{IN}\}$: movement actions are standard grid moves, while \textsf{IN} teleports from $i_p$ to $o_p$ at portal entrances and otherwise stays in place.

\paragraph{Datasets.}
Let $\pi_e^\star$ denote the optimal policy that knows the correct action at all portal entrances.
We consider two cases for suboptimal demonstration data.
In Case 1, each demonstrator $\pi_e^k$ has knowledge about $k$ portals, i.e., $\{(i_p, o_p)\}_{p=1}^k$.
This reflects the condition in Theorem~\ref{thm:bound1} that the quality of the added demonstrator increases with $K$. When $K=20$, we have an optimal demonstration in the dataset.
In Case 2, each demonstrator $\pi_e^k$ has knowledge about only one portal $p=k$, i.e., $(i_k, o_k)$. This reflects the condition in Theorem~\ref{thm:bound2} that the quality of each demonstrator is fixed while total coverage increases with $K$.
For portals unknown to a demonstrator, we assume that the demonstrator follows the optimal policy of the grid-world in which detour portals are ignored.
We illustrate the environment and the two cases of suboptimal demonstrators in Figure~\ref{fig:grid_world_illustration} in Appendix~\ref{app:exp_details}.

For each demonstrator $\pi_e^k$, we sample $N=100$ trajectories of horizon $H=40$ and compare the LP solution of the proposed method~\eqref{eq:objective} with a solution of the MaxEnt baseline described in Appendix~\ref{app:maxent}.
We omit the ranking-based baseline because it performs worse than both methods in this setting and requires strict ordering information between demonstrations, which our setup does not assume.
The performance of the learned reward is evaluated by the suboptimality of the optimal occupancy, i.e., $\subopt(\widehat r, d_e^\star)$, and the results are reported as the mean and standard deviation over $20$ random maps for each value of $K$.
Other hyperparameters are $\gamma=0.95, \epsilon^k=0.1, \beta_k = 1.0$ for all $k$.

\paragraph{Results.}
The results are shown in Figure \ref{fig:grid_world}.
In Case 1, both methods benefit from clearer quality differences, but the proposed method improves faster: its mean suboptimality decreases from $0.428$ at $K=1$ to $0.002$ at $K=20$, whereas MaxEnt decreases from $0.594$ to $0.330$.
In Case 2, the gap is more pronounced: the MaxEnt baseline remains nearly flat, with mean suboptimality $0.564$ at $K=1$ and $0.550$ at $K=20$, as the reward is fitted to the aggregated data with only highly suboptimal demonstrators.
By contrast, the proposed feasible-set method decreases from $0.389$ at $K=1$ to $0.043$ at $K=20$, which is predicted by Theorem~\ref{thm:bound2}.
Overall, these results confirm that our framework recovers substantially better rewards from suboptimal demonstrations than the MaxEnt baseline.

\begin{figure}
    \centering
    \begin{subfigure}[b]{0.45\textwidth}
    \centering
    \includegraphics[width=\textwidth]{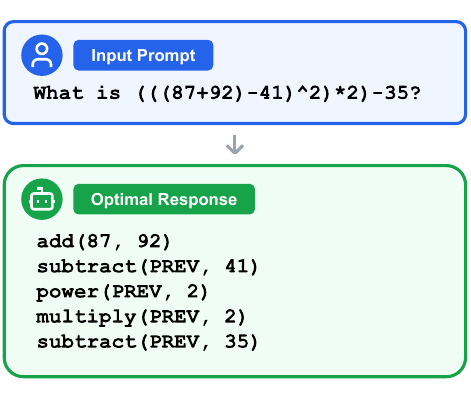}
    \end{subfigure}
    \hspace{0.02\textwidth}
    \begin{subfigure}[b]{0.48\textwidth}
    \centering
    \includegraphics[width=\textwidth]{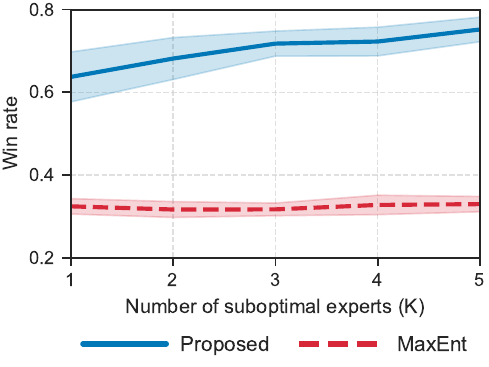}
    \end{subfigure}
    \caption{LLM experiment on arithmetic tool-use tasks. Left: Example of prompt and the optimal response. Right: Win rate of optimal policy under the trained advantage model.}
    \label{fig:llm_experiment}
\end{figure}

\subsection{LLM fine-tuning on arithmetic tool-use tasks}

\paragraph{Arithmetic tool-use tasks.}
We evaluate the function approximation algorithm from Section~\ref{sec:3:3} in the LLM domain with arithmetic tool-use tasks.
Each task consists of a single-turn arithmetic prompt whose answer is obtained through a chain of $1$--$5$ calls to five arithmetic tools (\texttt{add}, \texttt{subtract}, \texttt{multiply}, \texttt{divide}, \texttt{power}).
An optimal response is a newline-separated trace with one call per line, using a dedicated \texttt{PREV} token when an operand is the output of the previous line (Figure~\ref{fig:llm_experiment} left).

\paragraph{Datasets and evaluation.}
We consider $5$ suboptimal demonstrators, where each demonstrator has knowledge about one tool, and generates an incorrect (dummy) tool call on all other tools; increasing $K$ from $1$ to $5$ increases coverage of the tool set while keeping per-demonstrator quality low.
We used Qwen2.5-1.5B~\citep{yang2025qwen} with an additional head for advantage model and trained it using the proposed method with the advantage-parameterized loss \eqref{eq:loss_advantage_param} and the MaxEnt baseline (Appendix~\ref{app:maxent}).
Since the direct evaluation of the suboptimality measure can be blurred by an extra policy-optimization stage (for computing an optimal policy) and is sensitive to the reward (or advantage) scale, instead we report the win-rate of the optimal response with respect to the near-optimal response (which has at most one error in tool calling) as a proxy.
Additional details are provided in Appendix~\ref{app:exp_details}.

\paragraph{Results.}
Figure~\ref{fig:llm_experiment} shows the mean and standard deviation of the win-rate depending on the number of demonstrators $K$, over $10$ runs.
The MaxEnt baseline achieves an average win-rate of only $0.323$, below the $0.5$ win-rate of a constant reward, because it fits the reward to aggregated data from only suboptimal demonstrators.
The proposed method improves the win-rate (average win-rate $0.702$) over MaxEnt, increasing from $0.637$ at $K=1$ to $0.752$ at $K=5$, aligning with the recovery guarantee in Theorem~\ref{thm:bound2} and with the tabular experiments above.
These results suggest that our framework remains effective beyond tabular domains when implemented with the proposed function approximation algorithm.

\section{Conclusion}\label{sec:conclusion}
This work develops a feasible-set-based IRL framework for learning from multiple imperfect demonstrators without assuming access to an optimal demonstrator. By encoding demonstrator suboptimality as linear constraints and intersecting them, the approach ensures monotonic refinement of the candidate reward set and provides theoretical conditions for exact recovery under sufficient coverage. Experiments in both tabular and high-dimensional settings support these predictions. Overall, the framework opens a new avenue for principled IRL beyond likelihood-based methods, particularly in realistic scenarios with imperfect and heterogeneous demonstrations.



{
\small
\bibliographystyle{unsrtnat}
\bibliography{reference}

@article{ouyang2022training,
  title={Training language models to follow instructions with human feedback},
  author={Ouyang, Long and Wu, Jeffrey and Jiang, Xu and Almeida, Diogo and Wainwright, Carroll and Mishkin, Pamela and Zhang, Chong and Agarwal, Sandhini and Slama, Katarina and Ray, Alex and others},
  journal={Advances in neural information processing systems},
  volume={35},
  pages={27730--27744},
  year={2022}
}

@article{ziegler2019fine,
  title={Fine-tuning language models from human preferences},
  author={Ziegler, Daniel M and Stiennon, Nisan and Wu, Jeffrey and Brown, Tom B and Radford, Alec and Amodei, Dario and Christiano, Paul and Irving, Geoffrey},
  journal={arXiv preprint arXiv:1909.08593},
  year={2019}
}

@article{christiano2017deep,
  title={Deep reinforcement learning from human preferences},
  author={Christiano, Paul F and Leike, Jan and Brown, Tom and Martic, Miljan and Legg, Shane and Amodei, Dario},
  journal={Advances in neural information processing systems},
  volume={30},
  year={2017}
}

@article{park2024rlhf,
  title={Rlhf from heterogeneous feedback via personalization and preference aggregation},
  author={Park, Chanwoo and Liu, Mingyang and Kong, Dingwen and Zhang, Kaiqing and Ozdaglar, Asuman},
  journal={arXiv preprint arXiv:2405.00254},
  year={2024}
}

@inproceedings{kim2024unified,
  title={A Unified Linear Programming Framework for Offline Reward Learning from Human Demonstrations and Feedback},
  author={Kim, Kihyun and Zhang, Jiawei and Ozdaglar, Asuman E and Parrilo, Pablo},
  booktitle={International Conference on Machine Learning},
  pages={24694--24712},
  year={2024},
  organization={PMLR}
}

@article{yang2025qwen,
      title={Qwen2.5 Technical Report}, 
      author={An Yang and Baosong Yang and Beichen Zhang and Binyuan Hui and Bo Zheng and Bowen Yu and Chengyuan Li and Dayiheng Liu and Fei Huang and Haoran Wei and Huan Lin and Jian Yang and Jianhong Tu and Jianwei Zhang and Jianxin Yang and Jiaxi Yang and Jingren Zhou and Junyang Lin and Kai Dang and Keming Lu and Keqin Bao and Kexin Yang and Le Yu and Mei Li and Mingfeng Xue and Pei Zhang and Qin Zhu and Rui Men and Runji Lin and Tianhao Li and Tianyi Tang and Tingyu Xia and Xingzhang Ren and Xuancheng Ren and Yang Fan and Yang Su and Yichang Zhang and Yu Wan and Yuqiong Liu and Zeyu Cui and Zhenru Zhang and Zihan Qiu},
      journal={arXiv preprint arXiv:2412.15115},
      year={2024}
}

@inproceedings{ziebart2008maximum,
  title={Maximum entropy inverse reinforcement learning.},
  author={Ziebart, Brian D and Maas, Andrew L and Bagnell, J Andrew and Dey, Anind K and others},
  booktitle={Aaai},
  volume={8},
  pages={1433--1438},
  year={2008},
  organization={Chicago, IL, USA}
}

@inproceedings{ng2000algorithms,
  title={Algorithms for Inverse Reinforcement Learning},
  author={Ng, Andrew Y and Russell, Stuart J},
  booktitle={International Conference on Machine Learning},
  pages={663--670},
  year={2000},
  organization={PMLR}
}

@inproceedings{ramachandran2007bayesian,
  title={Bayesian Inverse Reinforcement Learning.},
  author={Ramachandran, Deepak and Amir, Eyal and others},
  booktitle={IJCAI},
  volume={7},
  pages={2586--2591},
  year={2007}
}

@article{poiani2024sub,
  title={Sub-optimal experts mitigate ambiguity in inverse reinforcement learning},
  author={Poiani, Riccardo and Curti, Gabriele and Metelli, Alberto M and Restelli, Marcello},
  journal={Advances in Neural Information Processing Systems},
  volume={37},
  pages={85778--85823},
  year={2024}
}

@article{lazzati2025generalizing,
  title={Generalizing Behavior via Inverse Reinforcement Learning with Closed-Form Reward Centroids},
  author={Lazzati, Filippo and Metelli, Alberto Maria},
  journal={arXiv preprint arXiv:2509.12010},
  year={2025}
}

@inproceedings{beliaev2025inverse,
  title={Inverse reinforcement learning by estimating expertise of demonstrators},
  author={Beliaev, Mark and Pedarsani, Ramtin},
  booktitle={Proceedings of the AAAI Conference on Artificial Intelligence},
  volume={39},
  pages={15532--15540},
  year={2025}
}

@article{hoffman1952approximate,
  title={On Approximate Solutions of Systems of Linear Inequalities},
  author={Hoffman, Alan J},
  journal={Journal of Research of the National Bureau of Standards},
  volume={49},
  number={4},
  year={1952}
}

@article{pena2021new,
  title={New characterizations of Hoffman constants for systems of linear constraints},
  author={Pena, Javier and Vera, Juan C and Zuluaga, Luis F},
  journal={Mathematical Programming},
  volume={187},
  number={1},
  pages={79--109},
  year={2021},
  publisher={Springer}
}

@inproceedings{abbeel2004apprenticeship,
  title={Apprenticeship learning via inverse reinforcement learning},
  author={Abbeel, Pieter and Ng, Andrew Y},
  booktitle={International Conference on Machine Learning},
  pages={1},
  year={2004}
}

@inproceedings{ratliff2006maximum,
  title={Maximum margin planning},
  author={Ratliff, Nathan D and Bagnell, J Andrew and Zinkevich, Martin A},
  booktitle={International Conference on Machine Learning},
  pages={729--736},
  year={2006}
}

@inproceedings{boularias2011relative,
  title={Relative Entropy Inverse Reinforcement Learning},
  author={Boularias, A and Kober, J and Peters, J},
  booktitle={Fourteenth International Conference on Artificial Intelligence and Statistics (AISTATS 2011)},
  pages={182--189},
  year={2011},
  organization={MIT Press}
}

@inproceedings{finn2016guided,
  title={Guided cost learning: Deep inverse optimal control via policy optimization},
  author={Finn, Chelsea and Levine, Sergey and Abbeel, Pieter},
  booktitle={International conference on machine learning},
  pages={49--58},
  year={2016},
  organization={PMLR}
}

@article{fu2017learning,
  title={Learning robust rewards with adversarial inverse reinforcement learning},
  author={Fu, Justin and Luo, Katie and Levine, Sergey},
  journal={arXiv preprint arXiv:1710.11248},
  year={2017}
}

@inproceedings{dimitrakakis2011bayesian,
  title={Bayesian multitask inverse reinforcement learning},
  author={Dimitrakakis, Christos and Rothkopf, Constantin A},
  booktitle={European Workshop on Reinforcement Learning},
  pages={273--284},
  year={2011},
  organization={Springer}
}

@inproceedings{babes2011apprenticeship,
  title={Apprenticeship learning about multiple intentions},
  author={Babes, Monica and Marivate, Vukosi and Subramanian, Kaushik and Littman, Michael L},
  booktitle={International Conference on Machine Learning},
  pages={897--904},
  year={2011}
}

@inproceedings{beliaev2022imitation,
  title={Imitation learning by estimating expertise of demonstrators},
  author={Beliaev, Mark and Shih, Andy and Ermon, Stefano and Sadigh, Dorsa and Pedarsani, Ramtin},
  booktitle={International Conference on Machine Learning},
  pages={1732--1748},
  year={2022},
  organization={PMLR}
}

@inproceedings{brown2019extrapolating,
  title={Extrapolating beyond suboptimal demonstrations via inverse reinforcement learning from observations},
  author={Brown, Daniel and Goo, Wonjoon and Nagarajan, Prabhat and Niekum, Scott},
  booktitle={International Conference on Machine Learning},
  pages={783--792},
  year={2019},
  organization={PMLR}
}

@inproceedings{brown2020better,
  title={Better-than-demonstrator imitation learning via automatically-ranked demonstrations},
  author={Brown, Daniel S and Goo, Wonjoon and Niekum, Scott},
  booktitle={Conference on Robot Learning},
  pages={330--359},
  year={2020},
  organization={PMLR}
}

@article{sikchi2023ranking,
  title={A Ranking Game for Imitation Learning},
  author={Sikchi, H and Saran, A and Goo, W and Niekum, S},
  journal={Transactions on machine learning research},
  year={2023}
}

@article{adams2022survey,
  title={A survey of inverse reinforcement learning},
  author={Adams, Stephen and Cody, Tyler and Beling, Peter A},
  journal={Artificial Intelligence Review},
  volume={55},
  number={6},
  pages={4307--4346},
  year={2022},
  publisher={Springer}
}

@inproceedings{chen2021learning,
  title={Learning from suboptimal demonstration via self-supervised reward regression},
  author={Chen, Letian and Paleja, Rohan and Gombolay, Matthew},
  booktitle={Conference on robot learning},
  pages={1262--1277},
  year={2021},
  organization={PMLR}
}

@inproceedings{jayanthi2023droid,
  title={Droid: Learning from offline heterogeneous demonstrations via reward-policy distillation},
  author={Jayanthi, Sravan and Chen, Letian and Balabanska, Nadya and Duong, Van and Scarlatescu, Erik and Ameperosa, Ezra and Zaidi, Zulfiqar Haider and Martin, Daniel and Del Matto, Taylor Keith and Ono, Masahiro and others},
  booktitle={Conference on Robot Learning},
  pages={1547--1571},
  year={2023},
  organization={PMLR}
}

@inproceedings{chen2020joint,
  title={Joint goal and strategy inference across heterogeneous demonstrators via reward network distillation},
  author={Chen, Letian and Paleja, Rohan and Ghuy, Muyleng and Gombolay, Matthew},
  booktitle={Proceedings of the 2020 ACM/IEEE international conference on human-robot interaction},
  pages={659--668},
  year={2020}
}

@article{bukharin2024robust,
  title={Robust reinforcement learning from corrupted human feedback},
  author={Bukharin, Alexander and Hong, Ilgee and Jiang, Haoming and Li, Zichong and Zhang, Qingru and Zhang, Zixuan and Zhao, Tuo},
  journal={Advances in Neural Information Processing Systems},
  volume={37},
  pages={124093--124113},
  year={2024}
}

@inproceedings{zhang2025diverging,
  title={Diverging Preferences: When do Annotators Disagree and do Models Know?},
  author={Zhang, Michael Jq and Wang, Zhilin and Hwang, Jena D and Dong, Yi and Delalleau, Olivier and Choi, Yejin and Choi, Eunsol and Ren, Xiang and Pyatkin, Valentina},
  booktitle={International Conference on Machine Learning},
  pages={76193--76212},
  year={2025},
  organization={PMLR}
}

@inproceedings{metelli2021provably,
  title={Provably efficient learning of transferable rewards},
  author={Metelli, Alberto Maria and Ramponi, Giorgia and Concetti, Alessandro and Restelli, Marcello},
  booktitle={International Conference on Machine Learning},
  pages={7665--7676},
  year={2021},
  organization={PMLR}
}

@inproceedings{metelli2023towards,
  title={Towards theoretical understanding of inverse reinforcement learning},
  author={Metelli, Alberto Maria and Lazzati, Filippo and Restelli, Marcello},
  booktitle={International Conference on Machine Learning},
  pages={24555--24591},
  year={2023},
  organization={PMLR}
}

@article{lindner2022active,
  title={Active exploration for inverse reinforcement learning},
  author={Lindner, David and Krause, Andreas and Ramponi, Giorgia},
  journal={Advances in Neural Information Processing Systems},
  volume={35},
  pages={5843--5853},
  year={2022}
}

@inproceedings{zhao2024inverse,
  title={Is Inverse Reinforcement Learning Harder than Standard Reinforcement Learning? A Theoretical Perspective},
  author={Zhao, Lei and Wang, Mengdi and Bai, Yu},
  booktitle={International Conference on Machine Learning},
  pages={60957--61020},
  year={2024},
  organization={PMLR}
}

@article{freihaut2024feasible,
  title={On feasible rewards in multi-agent inverse reinforcement learning},
  author={Freihaut, Till and Ramponi, Giorgia},
  journal={arXiv preprint arXiv:2411.15046},
  year={2024}
}

@article{ratliff2009learning,
  title={Learning to search: Functional gradient techniques for imitation learning},
  author={Ratliff, Nathan D and Silver, David and Bagnell, J Andrew},
  journal={Autonomous Robots},
  volume={27},
  number={1},
  pages={25--53},
  year={2009},
  publisher={Springer}
}

@book{puterman1994markov,
  title={Markov decision processes: discrete stochastic dynamic programming},
  author={Puterman, Martin L},
  year={1994},
  publisher={John Wiley \& Sons}
}

@book{shapiro2021lectures,
  title={Lectures on stochastic programming: modeling and theory},
  author={Shapiro, Alexander and Dentcheva, Darinka and Ruszczynski, Andrzej},
  year={2021},
  publisher={SIAM}
}

@inproceedings{bain1995framework,
  title={A Framework for Behavioural Cloning.},
  author={Bain, Michael and Sammut, Claude},
  booktitle={Machine intelligence 15},
  pages={103--129},
  year={1995}
}

@inproceedings{ross2010efficient,
  title={Efficient reductions for imitation learning},
  author={Ross, St{\'e}phane and Bagnell, Drew},
  booktitle={Proceedings of the thirteenth international conference on artificial intelligence and statistics},
  pages={661--668},
  year={2010},
  organization={JMLR Workshop and Conference Proceedings}
}

@inproceedings{ross2011reduction,
  title={A reduction of imitation learning and structured prediction to no-regret online learning},
  author={Ross, St{\'e}phane and Gordon, Geoffrey and Bagnell, Drew},
  booktitle={Proceedings of the fourteenth international conference on artificial intelligence and statistics},
  pages={627--635},
  year={2011},
  organization={JMLR Workshop and Conference Proceedings}
}

@article{ross2014reinforcement,
  title={Reinforcement and imitation learning via interactive no-regret learning},
  author={Ross, Stephane and Bagnell, J Andrew},
  journal={arXiv preprint arXiv:1406.5979},
  year={2014}
}

@article{ho2016generative,
  title={Generative adversarial imitation learning},
  author={Ho, Jonathan and Ermon, Stefano},
  journal={Advances in neural information processing systems},
  volume={29},
  year={2016}
}

@inproceedings{shiarlis2016inverse,
  title={Inverse Reinforcement Learning from Failure},
  author={Shiarlis, Kyriacos and Messias, Joao and Whiteson, Shimon},
  booktitle={Proceedings of the 2016 International Conference on Autonomous Agents \& Multiagent Systems},
  pages={1060--1068},
  year={2016}
}

@article{zhang2021confidence,
  title={Confidence-aware imitation learning from demonstrations with varying optimality},
  author={Zhang, Songyuan and Cao, Zhangjie and Sadigh, Dorsa and Sui, Yanan},
  journal={Advances in Neural Information Processing Systems},
  volume={34},
  pages={12340--12350},
  year={2021}
}

@article{gao2018reinforcement,
  title={Reinforcement learning from imperfect demonstrations},
  author={Gao, Yang and Xu, Huazhe and Lin, Ji and Yu, Fisher and Levine, Sergey and Darrell, Trevor},
  journal={arXiv preprint arXiv:1802.05313},
  year={2018}
}

@article{ibarz2018reward,
  title={Reward learning from human preferences and demonstrations in atari},
  author={Ibarz, Borja and Leike, Jan and Pohlen, Tobias and Irving, Geoffrey and Legg, Shane and Amodei, Dario},
  journal={Advances in neural information processing systems},
  volume={31},
  year={2018}
}

@article{syed2007game,
  title={A game-theoretic approach to apprenticeship learning},
  author={Syed, Umar and Schapire, Robert E},
  journal={Advances in neural information processing systems},
  volume={20},
  year={2007}
}

@inproceedings{syed2008apprenticeship,
  title={Apprenticeship learning using linear programming},
  author={Syed, Umar and Bowling, Michael and Schapire, Robert E},
  booktitle={Proceedings of the 25th international conference on Machine learning},
  pages={1032--1039},
  year={2008}
}

@article{bradley1952rank,
  title={Rank analysis of incomplete block designs: I. the method of paired comparisons},
  author={Bradley, Ralph Allan and Terry, Milton E},
  journal={Biometrika},
  volume={39},
  number={3/4},
  pages={324--345},
  year={1952},
  publisher={JSTOR}
}

@book{luce1959individual,
  title={Individual choice behavior},
  author={Luce, R Duncan and others},
  volume={4},
  year={1959},
  publisher={Wiley New York}
}

@article{plackett1975analysis,
  title={The analysis of permutations},
  author={Plackett, Robin L},
  journal={Journal of the Royal Statistical Society Series C: Applied Statistics},
  volume={24},
  number={2},
  pages={193--202},
  year={1975},
  publisher={Oxford University Press}
}
}

\newpage
\appendix
\section{Additional related work}\label{app:related}
\paragraph{IRL from a single optimal demonstrator and apprenticeship learning.}
Classical IRL studies the recovery of a reward function that rationalizes a single optimal demonstrator's behavior.
The foundational formulation of \citet{ng2000algorithms} characterizes rewards under which the demonstrated policy is optimal, thereby making explicit the non-identifiability of IRL and motivating solution concepts based on constraints.
\citet{abbeel2004apprenticeship} shifted the emphasis toward apprenticeship learning through feature expectation matching, while \citet{ratliff2006maximum} introduced max-margin formulations that select rewards or cost functions separating expert behavior from alternatives.
The game-theoretic and LP apprenticeship-learning formulations \citep{syed2007game, syed2008apprenticeship} are also relevant to our occupancy-measure viewpoint, as they cast learning from demonstrations as an optimization problem over policies and rewards.
Our contribution differs from this lineage in that we do not assume access to an optimal demonstrator to match or compete with; instead, we build a feasible reward set by intersecting slack-bounded constraints from multiple imperfect demonstrators.
Probabilistic approaches model expert stochasticity more directly: Bayesian IRL~\citep{ramachandran2007bayesian} places a posterior over rewards, and MaxEnt IRL~\citep{ziebart2008maximum} resolves ambiguity by choosing the distribution over trajectories with highest entropy subject to matching demonstration statistics.
These likelihood-based ideas have since become a dominant template for scalable IRL, including scalable variants such as Relative Entropy IRL~\citep{boularias2011relative}, Guided Cost Learning~\citep{finn2016guided}, and adversarial IRL (AIRL)~\citep{fu2017learning}.

\paragraph{Feasible-set approaches in IRL.} 
Our work is most closely related to the feasible-reward-set view of IRL, which frames reward learning not as selecting a single maximum-likelihood reward, but as characterizing the full set of rewards consistent with the demonstrations through optimality constraints.
This perspective is implicit in the original IRL formulation~\citep{ng2000algorithms} and is developed explicitly and with theoretical precision in \citet{metelli2021provably, metelli2023towards}, which formulate the problem of estimating the feasible reward set from finite-horizon Bellman equations and analyze its sample complexity.
\citet{lindner2022active} further studies an efficient trajectory-sampling strategy that relaxes the need for a state-action generative model when estimating the feasible reward set.
In the offline setting, \citet{zhao2024inverse} and \citet{kim2024unified} analyze sample complexity using penalty-based and LP-based approaches, respectively.
\citet{freihaut2024feasible} studies the feasible reward set for multi-agent IRL in entropy-regularized Markov games.
\citet{poiani2024sub} show that incorporating additional suboptimal demonstrations can shrink the feasible set and reduce ambiguity.
More recently, \citet{lazzati2025generalizing} study how to select principled representatives, such as reward centroids, from the feasible region to improve generalization.
Our method fits squarely within this feasible-set tradition, but differs in two important respects: we consider the setting with multiple suboptimal demonstrations only and no optimal demonstrator, and we aim to recover the ground-truth reward set rather than merely shrink the feasible set.

\paragraph{IRL from multiple demonstrations via MLE-based approaches.}
A natural extension of single-optimal-demonstrator MaxEnt IRL to heterogeneous demonstrators is to assume a shared latent reward while introducing demonstrator-specific noise, confidence, or inverse-temperature parameters, then fit the reward by maximizing an aggregate likelihood over all demonstrations.
This line includes Bayesian multi-task IRL formulations that cluster demonstrations and infer task-specific rewards~\citep{dimitrakakis2011bayesian, babes2011apprenticeship}, as well as more recent methods that explicitly estimate demonstrator expertise so that more reliable demonstrators exert more influence on the learned reward~\citep{beliaev2022imitation, beliaev2025inverse}.
Relatedly, IRL from failure shows that successful and failed demonstrations can reduce reward ambiguity~\citep{shiarlis2016inverse}.
The main strength of this family is statistical efficiency and compatibility with standard probabilistic modeling, since heterogeneity is absorbed into soft-optimality parameters or latent clusters rather than hard constraints.
However, these methods still commit to a single point estimate chosen to explain the observed visitation frequencies, which can be problematic when all demonstrators are suboptimal and no optimal demonstrator is available: the failure mode emphasized in our main text.

\paragraph{Ranking and preference-based reward learning.}
Another line of work leverages ranking or margin information rather than assuming a probabilistic model for likelihood maximization.
Classical choice models such as Bradley--Terry~\citep{bradley1952rank}, Luce choice~\citep{luce1959individual}, and Plackett--Luce ranking models~\citep{plackett1975analysis} provide statistical foundations for pairwise and listwise preference likelihoods.
The idea of enforcing a margin constraint appears in early imitation-learning work such as maximum margin planning (MMP)~\citep{ratliff2006maximum, ratliff2009learning}, which imposes a margin in the reward-optimization problem so that the demonstrated policy attains higher expected reward than alternative policies.
In the suboptimal-data regime, this idea reappears in trajectory-ranking and preference-based formulations such as T-REX~\citep{brown2019extrapolating}, which learns a reward from approximately ranked demonstrations and can extrapolate beyond imperfect behavior, and D-REX~\citep{brown2020better}, which automatically generates rankings by perturbing a cloned policy.
More recently, the ranking-game framework of \citet{sikchi2023ranking} proposes a two-player formulation that jointly learns a reward and a policy from expert demonstrations and offline preferences.
Although these methods are closely related in spirit to our objective because they use comparative structure to rule out degenerate rewards, they typically rely on explicit preference or ranking information rather than the maximum suboptimality level considered here.
In practical settings, explicit preference or ranking information is often difficult to obtain because demonstrators may have different areas of expertise.

\paragraph{Imitation learning (IL).}
IL learns policies directly from demonstrations rather than first recovering an explicit reward function.
The simplest formulation is behavior cloning, which casts IL as supervised learning over expert state-action pairs but is prone to distribution shift in sequential decision-making~\citep{bain1995framework}.
Reduction-based and interactive methods such as SMILe and DAgger address the resulting compounding-error and covariate-shift issues by analyzing sequential prediction and querying the expert on learner-induced states~\citep{ross2010efficient,ross2011reduction}.
Value-aware extensions such as AggreVaTe incorporate expert cost-to-go information so that imitation reflects long-horizon decision quality~\citep{ross2014reinforcement}, while GAIL reframes imitation as adversarial occupancy-measure matching~\citep{ho2016generative}.
More directly related to our motivation, \citet{zhang2021confidence} learn confidence scores and policies from demonstrations with varying optimality, and \citet{gao2018reinforcement} study reinforcement learning from imperfect demonstrations using an actor-critic objective designed to surpass demonstrator performance.
These works address the practical problem of learning good policies from imperfect data; our focus is instead reward recovery and feasible-set ambiguity under declared suboptimality bounds.

\paragraph{Reinforcement learning from human feedback (RLHF).}
Our motivation from LLM raters connects the proposed setting to RLHF.
\citet{christiano2017deep} learn rewards from non-expert human pairwise preferences, \citet{ibarz2018reward} combine demonstrations and trajectory preferences in Atari, \citet{ziegler2019fine} fine-tune language models from human preferences, and \citet{ouyang2022training} establish the InstructGPT pipeline using demonstrations, rankings, reward modeling, and policy optimization.
These works use human comparisons to train reward models and policies at scale.
In our framework, such comparisons can impose additional linear constraints on the reward set, thereby reducing ambiguity, as discussed in Appendix~\ref{app:pairwise_performance_gap}.

\section{Proof of Proposition~\ref{prop:lp_formulation}}\label{app:prop:lp_formulation}
By Definition~\ref{def:feasible_reward_set},
\begin{equation}
    \mathcal R(\mathcal D)=\left\{r\in\Delta(S \times A)\mid \subopt(r,d_e^k)\le \epsilon^k \quad \forall k \in [K]\right\}.
\end{equation}
Thus, for a fixed reward vector $r$, it suffices to prove
\begin{align}
    &\subopt(r,d_e^k)\le \epsilon^k \quad \forall k \in [K] \\
    &\Longleftrightarrow\quad
    \exists v\in\mathbb R^{|S|}\quad \text{s.t.}\quad M^\top v\ge r,\;
    (1-\gamma)\mu_0^\top v-r^\top d_e^k \le \epsilon^k \quad \forall k \in [K].
\end{align}

Consider the LP formulation of an MDP
\begin{equation}
    \max_{d\in\Phi} r^\top d,
\end{equation}
whose dual is
\begin{equation}
    \min_{v\in\mathbb R^{|S|}} (1-\gamma)\mu_0^\top v
    \quad \text{s.t.} \quad M^\top v\ge r.
\end{equation}
By strong duality, these two problems have the same optimal value, and the dual optimum is attained. Let
\begin{equation}
    J^\star(r):=\min_{v:\,M^\top v\ge r}(1-\gamma)\mu_0^\top v
    =\max_{d\in\Phi} r^\top d.
\end{equation}
Then, for every $k\in[K]$,
\begin{equation}
    \subopt(r,d_e^k)\le \epsilon^k
    \quad\Longleftrightarrow\quad
    J^\star(r)\le \epsilon^k+r^\top d_e^k.
\end{equation}
Therefore,
\begin{equation}
    \subopt(r,d_e^k)\le \epsilon^k \quad \forall k \in [K]
    \quad\Longleftrightarrow\quad
    J^\star(r)\le \epsilon^k+r^\top d_e^k \quad \forall k \in [K].
\end{equation}
Since the dual optimum is attained, there exists $v^\star\in\mathbb R^{|S|}$ such that
\begin{equation}
    M^\top v^\star\ge r,
    \qquad
    (1-\gamma)\mu_0^\top v^\star=J^\star(r).
\end{equation}
This $v^\star$ depends only on $r$, not on $k$. Substituting $J^\star(r)=(1-\gamma)\mu_0^\top v^\star$ into the previous display yields
\begin{align}
    &\subopt(r,d_e^k)\le \epsilon^k \quad \forall k \in [K]\\
    &\Longleftrightarrow\quad
    \exists v^\star\in\mathbb R^{|S|}\quad \text{s.t.}\quad M^\top v^\star\ge r,\;
    (1-\gamma)\mu_0^\top v^\star-r^\top d_e^k \le \epsilon^k
    \quad \forall k \in [K],
\end{align}
which proves the forward implication.
Conversely, suppose there exists some $v\in\mathbb R^{|S|}$ satisfying
$M^\top v\ge r$ and
\begin{equation}
    (1-\gamma)\mu_0^\top v-r^\top d_e^k\le \epsilon^k
    \quad \forall k\in[K].
\end{equation}
Because $J^\star(r)$ is the minimum dual value over all such feasible $v$,
\begin{equation}
    J^\star(r)-r^\top d_e^k
    \le
    (1-\gamma)\mu_0^\top v-r^\top d_e^k
    \le
    \epsilon^k
    \quad \forall k\in[K].
\end{equation}
Thus $\subopt(r,d_e^k)\le \epsilon^k$ for every $k$, completing the proof.

\section{Proof of Theorem~\ref{thm:monotonicity}}\label{app:thm:monotonicity}
\begin{lemma}[Coverage gap from newly visited mass]\label{lem:coverage_gap}
Let $\mathcal Z_{\mathcal D}:=\bigcup_{k=1}^K \supp(d_e^k)$ and
\begin{equation}
    \alpha(d_e',\mathcal D):=\sum_{(s,a)\notin\mathcal Z_{\mathcal D}}d_e'(s,a).
\end{equation}
There exists a reward $r_{\mathcal Z}\in\mathcal R(\mathcal D)$ such that
\begin{equation}
    \subopt(r_{\mathcal Z},d_e')
    =
    \frac{\alpha(d_e',\mathcal D)}{|\mathcal Z_{\mathcal D}|}.
\end{equation}
\end{lemma}
\begin{proof}
The idea is to build a reward that makes all previously observed state-action pairs equally valuable and assigns zero reward elsewhere. Define
\begin{equation}
    r_{\mathcal Z}(s,a)
    :=
    \begin{cases}
        1/|\mathcal Z_{\mathcal D}|, & (s,a)\in \mathcal Z_{\mathcal D},\\
        0, & (s,a)\notin \mathcal Z_{\mathcal D}.
    \end{cases}
\end{equation}
Then $r_{\mathcal Z}\in\Delta(S \times A)$.
Since every previous demonstrator has support contained in $\mathcal Z_{\mathcal D}$ and every occupancy measure has unit mass,
\begin{equation}
    r_{\mathcal Z}^\top d_e^k
    =
    \frac{1}{|\mathcal Z_{\mathcal D}|}
    \quad \forall k\in[K].
\end{equation}
For any $d\in\Phi$, we also have $r_{\mathcal Z}^\top d\le 1/|\mathcal Z_{\mathcal D}|$.
Thus, $\subopt(r_{\mathcal Z},d_e^k)=0\le\epsilon^k$ for every $k$, and hence $r_{\mathcal Z}\in\mathcal R(\mathcal D)$.

The only way the new demonstrator loses value under this reward is by placing mass outside the old support. Therefore,
\begin{equation}
    r_{\mathcal Z}^\top d_e'
    =
    \frac{1-\alpha(d_e',\mathcal D)}{|\mathcal Z_{\mathcal D}|}.
\end{equation}
Since the previous demonstrators attain value $1/|\mathcal Z_{\mathcal D}|$ under $r_{\mathcal Z}$,
\begin{equation}
    \subopt(r_{\mathcal Z},d_e')
    =
    \max_{d\in\Phi} r_{\mathcal Z}^\top d-r_{\mathcal Z}^\top d_e'
    =
    \frac{1}{|\mathcal Z_{\mathcal D}|} - \frac{1-\alpha(d_e',\mathcal D)}{|\mathcal Z_{\mathcal D}|}
    =
    \frac{\alpha(d_e',\mathcal D)}{|\mathcal Z_{\mathcal D}|}.
\end{equation}
\end{proof}

\begin{proof}[Proof of Theorem~\ref{thm:monotonicity}]
By Definition~\ref{def:feasible_reward_set},
\begin{equation}
    \mathcal R(\mathcal D^+)
    =
    \mathcal R(\mathcal D)\cap \mathcal R(d_e';\epsilon').
\end{equation}
The augmented feasible set is therefore obtained by adding one more constraint to the old feasible set, so $\mathcal R(\mathcal D^+) \subseteq \mathcal R(\mathcal D)$.
Since Assumption~\ref{ass:consistency} holds for the augmented dataset $\mathcal D^+$, we also have $\mathcal R^\star \subseteq \mathcal R(\mathcal D^+)$.

By Lemma~\ref{lem:coverage_gap}, there exists $r_{\mathcal Z}\in\mathcal R(\mathcal D)$ such that
\begin{equation}
    \subopt(r_{\mathcal Z},d_e')
    =
    \frac{\alpha(d_e',\mathcal D)}{|\mathcal Z_{\mathcal D}|}
    >
    \epsilon'.
\end{equation}
Therefore, $r_{\mathcal Z}\notin \mathcal R(d_e';\epsilon')$.
Thus the added constraint removes at least one reward that was previously feasible, namely $r_{\mathcal Z}$. Hence $\mathcal R(\mathcal D^+) \subsetneq \mathcal R(\mathcal D)$.
Combining this with $\mathcal R^\star \subseteq \mathcal R(\mathcal D^+)$ proves the theorem.
\end{proof}

\begin{lemma}[Exact characterization of strict shrinkage]\label{lem:strict_shrinkage_characterization}
Suppose Assumption~\ref{ass:consistency} holds for the dataset $\mathcal{D}$ and the augmented dataset $\mathcal{D}^+ := \mathcal{D}\cup\{(d_e',\epsilon')\}$.
Then,
\begin{equation}
    \mathcal{R}^\star \subseteq \mathcal{R}(\mathcal{D}^+) \subsetneq \mathcal{R}(\mathcal{D})
\end{equation}
if and only if
\begin{equation}
    \epsilon' < \max_{r \in \mathcal{R}(\mathcal{D})} \subopt(r,d_e').
\end{equation}
\end{lemma}
\begin{proof}
The identity
\begin{equation}
    \mathcal R(\mathcal D^+)
    =
    \mathcal R(\mathcal D)\cap \mathcal R(d_e';\epsilon')
\end{equation}
implies $\mathcal R(\mathcal D^+) \subseteq \mathcal R(\mathcal D)$.
Moreover, Assumption~\ref{ass:consistency} for $\mathcal D^+$ gives $\mathcal R^\star \subseteq \mathcal R(\mathcal D^+)$.

It remains to identify exactly when the last inclusion is strict. First, assume
\begin{equation}
    \epsilon' < \max_{r \in \mathcal{R}(\mathcal D)} \subopt(r,d_e').
\end{equation}
Then there exists some $\tilde r \in \mathcal R(\mathcal D)$ such that
\begin{equation}
    \subopt(\tilde r,d_e')>\epsilon'.
\end{equation}
Therefore, $\tilde r\notin \mathcal R(d_e';\epsilon')$, and hence
\begin{equation}
    \tilde r \notin \mathcal R(\mathcal D)\cap \mathcal R(d_e';\epsilon')
    =
    \mathcal R(\mathcal D^+).
\end{equation}
So $\mathcal R(\mathcal D^+) \subsetneq \mathcal R(\mathcal D)$.
Together with $\mathcal R^\star \subseteq \mathcal R(\mathcal D^+)$, this yields the desired strict inclusion.

Conversely, assume
\begin{equation}
    \mathcal R^\star \subseteq \mathcal R(\mathcal D^+) \subsetneq \mathcal R(\mathcal D).
\end{equation}
Since the second inclusion is strict, there exists some $\hat r \in \mathcal R(\mathcal D)$ such that
\begin{equation}
    \hat r \notin \mathcal R(\mathcal D^+).
\end{equation}
Using $\mathcal R(\mathcal D^+) = \mathcal R(\mathcal D)\cap \mathcal R(d_e';\epsilon')$ and $\hat r \in \mathcal R(\mathcal D)$, we must have
\begin{equation}
    \hat r \notin \mathcal R(d_e';\epsilon').
\end{equation}
By Definition~\ref{def:feasible_reward_set}, this means
\begin{equation}
    \subopt(\hat r,d_e')>\epsilon'.
\end{equation}
Since $\hat r \in \mathcal R(\mathcal D)$, it follows that
\begin{equation}
    \max_{r \in \mathcal R(\mathcal D)} \subopt(r,d_e')
    \ge
    \subopt(\hat r,d_e')
    >
    \epsilon',
\end{equation}
which proves the reverse implication.
\end{proof}

\section{One-sided Hausdorff distance and maximum suboptimality}\label{app:hausdorff_distance}

\begin{definition}[One-sided Hausdorff distance]
For two nonempty sets $A,B \subset \mathbb{R}^m$, the one-sided Hausdorff distance from $A$ to $B$ with respect to the $\ell_\infty$ norm is defined as
\begin{equation}
d_H(A\mid B)
:=
\sup_{x \in A} \inf_{y \in B} \|x-y\|_\infty.
\end{equation}
\end{definition}

\begin{proposition}
There exists a constant $C>0$ depending only on the reward-free MDP $\mathcal M$ such that
\begin{equation}
d_H\bigl(\mathcal R(\mathcal D)\mid \mathcal R(d_e^\star)\bigr)
\le
C\,\Gap\bigl(\mathcal R(\mathcal D),\mathcal R(d_e^\star)\bigr).
\end{equation}
In particular, $\Gap(\mathcal R(\mathcal D),\mathcal R(d_e^\star))\to 0$ implies that $\mathcal R(\mathcal D)\to\mathcal R(d_e^\star)$ in the one-sided Hausdorff distance induced by the $\ell_\infty$ norm.
\end{proposition}
\begin{proof}
Since $\Phi$ is a polytope, $\mathcal R(d_e^\star)$ can equivalently be written as the simplex intersected with the finitely many inequalities
$r^\top(d-d_e^\star)\le 0$ for the vertices $d$ of $\Phi$.
The Hoffman error bound states that, for a finite feasible system of linear inequalities, the distance to the feasible set is bounded by a constant times the maximum constraint violation~\cite{hoffman1952approximate, pena2021new}.
Therefore, there exists $C>0$ such that for all $r \in \Delta(S \times A)$,
\begin{equation}
\inf_{r'\in \mathcal R(d_e^\star)} \|r-r'\|_\infty
\le
C \max_{d \in \Phi} \bigl(r^\top(d-d_e^\star)\bigr)_+
=
C\,\subopt(r,d_e^\star).
\end{equation}

Taking the supremum over $r \in \mathcal R(\mathcal D)$ gives
\begin{equation}
d_H\bigl(\mathcal R(\mathcal D)\mid\mathcal R(d_e^\star)\bigr)
\le
C\,\Gap\bigl(\mathcal R(\mathcal D),\mathcal R(d_e^\star)\bigr).
\end{equation}
Therefore, if $\Gap(\mathcal R(\mathcal D),\mathcal R(d_e^\star)) \to 0$, then $d_H\bigl(\mathcal R(\mathcal D)\mid\mathcal R(d_e^\star)\bigr) \to 0$.
\end{proof}

\section{Optimal-face recovery from relative-interior rewards}\label{app:optimal_face_recovery}
We state the standard geometric fact underlying the face-recovery statement in Section~\ref{sec:3}.
For a face $F$ of $\Phi$, define its normal cone by
\begin{equation}
N_\Phi(F)
:=
\left\{ r \mid r^\top d = r^\top d' \;\; \forall d,d'\in F,\quad r^\top y \le r^\top d \;\; \forall y\in\Phi,\; \forall d\in F \right\}.
\end{equation}
Let $F_e$ denote the minimal face of $\Phi$ containing $d_e$.
Then, we prove that relative interior points of the feasible reward set expose the expert face $F_e$.

\begin{lemma}\label{lem:relative_interior_face_recovery}
For any $r\in\mathcal R(d_e)$, the optimal occupancy set $\Phi^\star(r)$ is an exposed face of $\Phi$ containing $d_e$.
Moreover, if $r\in\relint(\mathcal R(d_e))$, then $\Phi^\star(r)=F_e$.
\end{lemma}
\begin{proof}
If $r\in\mathcal R(d_e)$, then $d_e\in\Phi^\star(r)$ by definition.
Since $\Phi^\star(r)=\argmax_{d\in\Phi} r^\top d$ is the maximizer set of a linear functional over the polytope $\Phi$, it is an exposed face of $\Phi$.
Thus $\Phi^\star(r)$ is an exposed face containing $d_e$.

It remains to prove the second claim.
Let $\Delta:=\Delta(S\times A)$.
Since $F_e$ is the minimal face containing $d_e$, we have
\begin{equation}
\mathcal R(d_e)=N_\Phi(F_e)\cap\Delta .
\end{equation}
To apply the relative-interior intersection rule, we first verify that
\(\relint(N_\Phi(F_e))\) intersects \(\relint(\Delta)\).
Choose any \(u\in\relint(N_\Phi(F_e))\).
Every occupancy measure in \(\Phi\) has total mass one, so adding a constant vector to a reward does not change its maximizers over \(\Phi\).
Thus \(u+\alpha\mathbf 1\) remains in \(\relint(N_\Phi(F_e))\) for any \(\alpha\), and choosing \(\alpha>0\) sufficiently large makes all coordinates positive.
After normalizing, we obtain
\begin{equation}
\frac{u+\alpha\mathbf 1}{\mathbf 1^\top(u+\alpha\mathbf 1)}
\in
\relint(N_\Phi(F_e))\cap\relint(\Delta).
\end{equation}
Therefore \(\relint(N_\Phi(F_e))\cap\relint(\Delta)\neq\emptyset\), and the standard relative-interior intersection rule for polyhedra yields
\begin{equation}
\relint(\mathcal R(d_e))
=
\relint(N_\Phi(F_e))\cap\relint(\Delta).
\end{equation}
Thus $r\in\relint(\mathcal R(d_e))$ implies $r\in\relint(N_\Phi(F_e))$.
By the definition of the normal cone, every vector in $N_\Phi(F_e)$ is constant on $F_e$ and satisfies the weak inequalities
\begin{equation}
r^\top y \le r^\top d
\qquad
\forall d\in F_e,\ \forall y\in\Phi .
\end{equation}
We now show that these inequalities are strict for points outside $F_e$.
Fix $d\in F_e$ and $y\in\Phi\setminus F_e$.
Since $F_e$ is a face of the polytope $\Phi$, there exists some $\bar r\in N_\Phi(F_e)$ that exposes $F_e$ and therefore satisfies
\[
\bar r^\top y < \bar r^\top d .
\]
Suppose, toward a contradiction, that our relative-interior point $r\in\relint(N_\Phi(F_e))$ satisfies $r^\top y=r^\top d$.
Because $r$ lies in the relative interior of the cone, the point
\[
r+\varepsilon(r-\bar r)
\]
also lies in $N_\Phi(F_e)$ for all sufficiently small $\varepsilon>0$.
However,
\[
\bigl(r+\varepsilon(r-\bar r)\bigr)^\top y
>
\bigl(r+\varepsilon(r-\bar r)\bigr)^\top d,
\]
contradicting the defining inequality of the normal cone.
Therefore,
\begin{equation}
r^\top y < r^\top d
\qquad
\forall d\in F_e,\ \forall y\in\Phi\setminus F_e.
\end{equation}
Therefore the maximizers of $r^\top d$ over $\Phi$ are precisely the points in $F_e$, i.e., $\Phi^\star(r)=F_e$.
\end{proof}

Consequently, if $\Gap(\mathcal R(\mathcal D),\mathcal R(d_e^\star))=0$, then every selected reward in $\mathcal R(\mathcal D)$ exposes an optimal face containing $d_e^\star$.
If the selected reward additionally lies in $\relint(\mathcal R(d_e^\star))$, then policy optimization recovers exactly the minimal face.

\section{Proof of Theorem~\ref{thm:bound1}}\label{app:thm:bound1}
Fix any $\tilde r \in \mathcal R(\mathcal D)$ and any $(d,\epsilon)\in \conv(\mathcal D)$.
By definition of the convex hull, there exist coefficients $\{\lambda_k\}_{k=1}^K$ such that
\begin{equation}
    \lambda_k \ge 0,\qquad \sum_{k=1}^K \lambda_k = 1,\qquad
    d = \sum_{k=1}^K \lambda_k d_e^k,\qquad
    \epsilon = \sum_{k=1}^K \lambda_k \epsilon^k.
\end{equation}
Since $\tilde r \in \mathcal R(\mathcal D)$, we have $\subopt(\tilde r,d_e^k)\le \epsilon^k$ for every $k\in[K]$.
Using the definition of $\subopt$ and the linearity in the second argument,
\begin{align}
    \subopt(\tilde r,d)
    &= \max_{\tilde d\in\Phi}\tilde r^\top(\tilde d-d) \\
    &= \max_{\tilde d\in\Phi}\tilde r^\top \tilde d - \tilde r^\top \sum_{k=1}^K \lambda_k d_e^k \\
    &= \sum_{k=1}^K \lambda_k \left(\max_{\tilde d\in\Phi}\tilde r^\top \tilde d - \tilde r^\top d_e^k\right) \\
    &= \sum_{k=1}^K \lambda_k \subopt(\tilde r,d_e^k)
    \le \sum_{k=1}^K \lambda_k \epsilon^k
    = \epsilon.
\end{align}
Next, we compare the suboptimality of $d_e^\star$ with that of the convex-combined demonstrator occupancy $d$:
\begin{align}
    \subopt(\tilde r,d_e^\star)
    &= \max_{\tilde d\in\Phi}\tilde r^\top(\tilde d-d_e^\star) \\
    &= \max_{\tilde d\in\Phi}\tilde r^\top(\tilde d-d) + \tilde r^\top(d-d_e^\star) \\
    &\le \subopt(\tilde r,d) + \left|\tilde r^\top(d-d_e^\star)\right| \\
    &\le \epsilon + \|\tilde r\|_\infty \|d-d_e^\star\|_1 \\
    &\le \epsilon + \|d-d_e^\star\|_1,
\end{align}
where the last inequality uses $\tilde r\in [0,1]^{|S||A|}$.
Since $(d,\epsilon)\in\conv(\mathcal D)$ was arbitrary, we obtain
\begin{equation}
    \subopt(\tilde r,d_e^\star)
    \le
    \min_{(d,\epsilon)\in\conv(\mathcal D)}
    \left\{\|d-d_e^\star\|_1+\epsilon\right\}.
\end{equation}
Finally, taking the maximum over $\tilde r\in \mathcal R(\mathcal D)$ yields
\begin{equation}
    \Gap(\mathcal R(\mathcal D),\mathcal R(d_e^\star))
    =
    \max_{\tilde r\in \mathcal R(\mathcal D)} \subopt(\tilde r,d_e^\star)
    \le
    \min_{(d,\epsilon)\in\conv(\mathcal D)}
    \left\{\|d-d_e^\star\|_1+\epsilon\right\},
\end{equation}
which proves the theorem.

\section{Proof of Theorem~\ref{thm:bound2}}\label{app:thm:bound2}
\begin{lemma}\label{lemma:bound2}
Consider any reward $\tilde r$ with $\subopt(\tilde r,d_e^\star) > \delta$.
Suppose that there exists $r' \in \mathcal{R}(d_e^\star)$ and $(d_e^k, \epsilon^k) \in \mathcal{D}$ such that
\begin{equation}
    (\tilde r - r')^\top (d_e^\star - d_e^k) \geq 0, \quad \text{and} \quad \subopt(r',d_e^k) \in [\epsilon^k - \delta, \epsilon^k].
\end{equation}
Then, $\subopt(\tilde r,d_e^{k}) > \epsilon^{k}$.    
\end{lemma}
\begin{proof}
The key point is to decompose the value gap against demonstrator $k$ through the optimal occupancy $d_e^\star$ and the reference reward $r'$. Since $r'\in\mathcal R(d_e^\star)$, $d_e^\star$ is optimal under $r'$, and hence
\begin{equation}
    r'^\top(d_e^\star-d_e^k)=\subopt(r',d_e^k).
\end{equation}
Using this identity,
\begin{align}
    \subopt(\tilde r,d_e^{k}) &= \max_{\tilde d \in \Phi} \tilde r^\top (\tilde d - d_e^{k})\\
    & = \max_{\tilde d \in \Phi} \tilde r^\top (\tilde d - d_e^\star)  + (\tilde r - r')^\top (d_e^\star - d_e^{k}) + r'^\top (d_e^\star - d_e^{k}) \\
    &= \subopt(\tilde r,d_e^\star) + (\tilde r - r')^\top (d_e^\star - d_e^{k}) + \subopt(r',d_e^{k}) \\
    &> \delta + \epsilon^{k} - \delta = \epsilon^{k}.
\end{align}
\end{proof}

Next, we prove Theorem~\ref{thm:bound2}.
Assume for contradiction that $\Gap(\mathcal{R}(\mathcal D), \mathcal{R}(d_e^\star)) > \delta$.
Then, there exists $\tilde r\in \mathcal{R}(\mathcal D)$ such that $\subopt(\tilde r,d_e^\star)>\delta$.
By the assumption of the theorem, there exist $r' \in \mathcal{R}(d_e^\star)$ and $(d_e^{k}, \epsilon^{k}) \in \mathcal{D}$ such that
\begin{equation}
    (\tilde r-r')^\top(d_e^\star-d_e^{k}) > 0, \quad \text{and} \quad
    \subopt(r', d_e^{k}) \in [\epsilon^{k} - \delta, \epsilon^{k}].
\end{equation}
Applying Lemma~\ref{lemma:bound2}, we obtain
\begin{equation}
    \subopt(\tilde r,d_e^{k}) > \epsilon^{k}.
\end{equation}
This violates the $k$th feasibility constraint in the definition of $\mathcal R(\mathcal D)$: every reward $r\in\mathcal R(\mathcal D)$ must satisfy $\subopt(r,d_e^j)\le \epsilon^j$ for all $j\in[K]$.
Thus $\tilde r\notin\mathcal R(\mathcal D)$, contradicting the choice of $\tilde r$.
Hence no such $\tilde r$ exists, and therefore $\Gap(\mathcal{R}(\mathcal D), \mathcal{R}(d_e^\star)) \leq \delta$.

\section{Finite demonstrators for Theorem~\ref{thm:bound2}}\label{app:finite_demonstrators}
The pointwise condition in Theorem~\ref{thm:bound2} quantifies over infinitely many rewards.
However, if the same condition is required on the closed set of rewards with suboptimality at least $\delta$, then compactness implies that finitely many demonstrators suffice.

\begin{proposition}[Finite demonstrators]\label{prop:finite_demonstrators}
Assume $\delta>0$.
Define
\begin{equation}
    \mathcal K_\delta
    :=
    \left\{\tilde r\in\Delta(S \times A)\mid \subopt(\tilde r,d_e^\star)\ge \delta\right\}.
\end{equation}
Suppose that for every $\bar r\in\mathcal K_\delta$, there exist $r_{\bar r}\in\mathcal R(d_e^\star)$ and a demonstrator pair $(d_{\bar r},\epsilon_{\bar r})$ such that
\begin{equation}\label{eq:local_condition}
    (\bar r-r_{\bar r})^\top(d_e^\star-d_{\bar r})>0,
    \qquad
    \subopt(r_{\bar r},d_{\bar r})\in[\epsilon_{\bar r}-\delta,\epsilon_{\bar r}].
\end{equation}
Then there exists a finite subset $\{\bar r_i\}_{i=1}^N\subseteq\mathcal K_\delta$ such that the finite collection
\begin{equation}
    \mathcal C_{\mathrm{fin}}
    :=
    \{(d_{\bar r_i},\epsilon_{\bar r_i})\}_{i=1}^N
\end{equation}
satisfies the pointwise condition of Theorem~\ref{thm:bound2} for every $\tilde r\in\mathcal K_\delta$.
In particular, it satisfies the condition for every $\tilde r$ with $\subopt(\tilde r,d_e^\star)>\delta$.
\end{proposition}

\begin{proof}
The set $\mathcal K_\delta$ is compact because $\Delta(S \times A)$ is compact and $\subopt(\cdot,d_e^\star)$ is continuous.
For each $\bar r\in\mathcal K_\delta$, choose one triple $(r_{\bar r},d_{\bar r},\epsilon_{\bar r})$ satisfying \eqref{eq:local_condition}.
Since the first inequality in \eqref{eq:local_condition} is strict and depends continuously on the reward argument, the set
\begin{equation}
    U_{\bar r}
    :=
    \left\{
    \tilde r\in\Delta(S \times A)
    \mid
    (\tilde r-r_{\bar r})^\top(d_e^\star-d_{\bar r})>0
    \right\}
\end{equation}
is an open neighborhood of $\bar r$.
Moreover, the same triple $(r_{\bar r},d_{\bar r},\epsilon_{\bar r})$ satisfies the two conditions in Theorem~\ref{thm:bound2} for every $\tilde r\in U_{\bar r}$, because the second condition is independent of $\tilde r$.

The collection $\{U_{\bar r}\}_{\bar r\in\mathcal K_\delta}$ is an open cover of $\mathcal K_\delta$.
By compactness, there exist $\bar r_1,\ldots,\bar r_N\in\mathcal K_\delta$ such that
\begin{equation}
    \mathcal K_\delta\subseteq \bigcup_{i=1}^N U_{\bar r_i}.
\end{equation}
Now fix any $\tilde r\in\mathcal K_\delta$.
Choose $i$ such that $\tilde r\in U_{\bar r_i}$.
Then
\begin{equation}
    (\tilde r-r_{\bar r_i})^\top(d_e^\star-d_{\bar r_i})>0,
\end{equation}
and by construction
\begin{equation}
    \subopt(r_{\bar r_i},d_{\bar r_i})
    \in
    [\epsilon_{\bar r_i}-\delta,\epsilon_{\bar r_i}].
\end{equation}
Thus the finite collection $\mathcal C_{\mathrm{fin}}$ satisfies the pointwise condition of Theorem~\ref{thm:bound2} on $\mathcal K_\delta$.
\end{proof}

\section{Multi-demonstrator MaxEnt baseline}\label{app:maxent}
\paragraph{Algorithm.}
A standard maximum-causal-entropy IRL baseline assumes that all demonstrators share one latent reward $r$, while demonstrator $k$ acts according to the entropy-regularized optimal policy with inverse temperature $\beta_k>0$.
In the baseline analyzed here, the inverse temperatures are fixed before reward fitting, possibly as deterministic functions $\beta_k=B(\epsilon^k)$ of the declared suboptimality levels; the likelihood is then optimized only over $r$.
Thus, demonstrators with the same declared suboptimality level share the same inverse temperature, while demonstrators with different levels may have different temperatures.
This is the usual MaxEnt / maximum-causal-entropy viewpoint~\citep{ziebart2008maximum,ramachandran2007bayesian}, with $\beta_k$ controlling the strength of entropy regularization and capturing different suboptimality levels.
Equivalently, for fixed $r$ and $\beta_k$, demonstrator $k$ solves the entropy-regularized control problem whose soft Bellman equations define the state-action and state values $Q^k_{r}(s,a)$ and $V^k_{r}(s)$:
\begin{equation}
\begin{split}
V^k_r(s)
&:= \frac{1}{\beta_k}\log\sum_{a\in A}\exp\!\big(\beta_k Q^k_{r}(s,a)\big),\\
Q^k_r(s,a)
&:= r(s,a) + \gamma \sum_{s'\in S} P(s'\mid s,a)\,V^k_{r}(s').
\end{split}
\end{equation}
The induced policy is the Boltzmann distribution
\begin{equation}
\pi^k_r(a\mid s)\propto \exp(\beta_k Q^k_r(s,a)).
\end{equation}
Given $\mathcal D=\{(d_e^k,\epsilon^k)\}_{k=1}^K$, the MaxEnt estimator is
\begin{equation}
\hat r \in \argmax_{r\in \Delta(S \times A)} \mathcal{L}(r),
\qquad
\mathcal{L}(r):=\sum_{k=1}^K \mathbb{E}_{(s,a)\sim d_e^k}\!\left[\log \pi_r^k(a\mid s)\right].
\end{equation}

\paragraph{Proof of Proposition~\ref{prop:baseline-failure} for MaxEnt.}
Consider a one-state, two-action bandit with action set $A=\{a_1,a_2\}$.
Let the true optimal occupancy be $d_e^\star=e_{a_1}$, where $e_a$ denotes the unit vector for action $a$, and define the reward gap
\begin{equation}
\Delta:=r(a_1)-r(a_2)\in[-1,1].
\end{equation}

Fix any $\eta>0$ and any constant $C\in(0,1)$.
Choose $\delta:=\min\{\eta/3,1/4\}$.
Let $\epsilon^g=\delta$ and $\epsilon^b=1$, and let
\begin{equation}
\beta_g:=B(\epsilon^g),
\qquad
\beta_b:=B(\epsilon^b)
\end{equation}
be the fixed inverse temperatures assigned by the baseline.
Choose an integer $m$ such that
\begin{equation}
m\ge
\frac{\beta_g \sigma(\beta_g C)}
{\beta_b \sigma(-\beta_b C)}
,
\qquad
\sigma(x):=(1+e^{-x})^{-1}.
\end{equation}
Define the finite dataset
\begin{equation}
\mathcal D:=\{(d^g,\epsilon^g)\}\cup\{(d^{b,i},\epsilon^b)\}_{i=1}^m,
\end{equation}
where
\begin{equation}
d^g=(1-\delta,\delta),
\end{equation}
and
\begin{equation}
d^{b,i}=(0,1),\qquad i=1,\ldots,m.
\end{equation}
Since $\delta>0$, the dataset contains no optimal demonstrator.

This dataset is consistent with $\mathcal R^\star$.
Indeed, if $r\in\mathcal R^\star$, then $a_1$ is optimal under $r$, so $\Delta\in[0,1]$.
Hence
\begin{equation}
\subopt(r,d^g)=\delta \Delta\le \delta=\epsilon^g,
\end{equation}
and for each bad demonstrator,
\begin{equation}
\subopt(r,d^{b,i})=\Delta\le 1=\epsilon^b.
\end{equation}
Therefore $\mathcal R^\star\subseteq \mathcal R(\mathcal D)$.

Since $(d^g,\epsilon^g)\in\mathcal D\subseteq\conv(\mathcal D)$, Theorem~\ref{thm:bound1} gives
\begin{equation}
\min_{(d,\epsilon)\in\conv(\mathcal D)}\left\{\|d-d_e^\star\|_1+\epsilon\right\}
\le
\|d^g-d_e^\star\|_1+\epsilon^g
=
3\delta
\le
\eta.
\end{equation}

We next show that every MaxEnt optimum for this finite dataset gives $d_e^\star$ a constant suboptimality gap.
For the dataset $\mathcal D$, the one-dimensional log-likelihood is
\begin{equation}
L_m(\Delta)
:=
\left(1-\delta\right)\log\sigma(\beta_g\Delta)
+\delta\log\sigma(-\beta_g\Delta)
+m\log\sigma(-\beta_b\Delta).
\end{equation}
Its derivative is
\begin{equation}
L_m'(\Delta)
=
\beta_g\left((1-\delta)\sigma(-\beta_g\Delta)-\delta\sigma(\beta_g\Delta)\right)
-m\beta_b\sigma(\beta_b\Delta),
\end{equation}
and $L_m$ is concave in $\Delta$.
At $\Delta=-C$,
\begin{equation}
L_m'(-C)
\le
\beta_g\sigma(\beta_g C)-m\beta_b\sigma(-\beta_b C)
\le 0
\end{equation}
by the choice of $m$.
Since $L_m$ is concave, every maximizer $\Delta^\star$ satisfies $\Delta^\star\le -C$.
Hence every MaxEnt estimator $\hat r$ satisfies
\begin{equation}
\hat r(a_2)-\hat r(a_1)\ge C.
\end{equation}
Therefore,
\begin{equation}
\subopt(\hat r,d_e^\star)=\bigl(\hat r(a_2)-\hat r(a_1)\bigr)_+\ge C.
\end{equation}
Thus, for any fixed $C\in(0,1)$, we have constructed a finite dataset $\mathcal D$ for which the upper bound in Theorem~\ref{thm:bound1} is at most $\eta$, while every MaxEnt estimator satisfies $\subopt(\hat r,d_e^\star)\ge C$.

\section{Ranking-based baseline}\label{app:ranking}
\paragraph{Algorithm.}
A standard ranking-based baseline such as T-REX learns a reward from pairwise trajectory rankings.
Given ranked trajectories $\{\tau^k\}_{k=1}^n$, T-REX assigns trajectory $\tau$ the cumulative reward
\begin{equation}
r(\tau):=\sum_{t=0}^{T(\tau)-1} r(s_t,a_t),
\end{equation}
and models the preference probability by a Bradley--Terry--Luce model~\citep{bradley1952rank,luce1959individual} or a Plackett--Luce ranking model~\citep{plackett1975analysis}:
\begin{equation}
\mathbb P_r(\tau^i \prec \tau^j)
:=
\frac{\exp(r(\tau^j))}{\exp(r(\tau^i))+\exp(r(\tau^j))}.
\end{equation}
The corresponding reward estimator is
\begin{equation}
\hat r \in \argmax_{r\in\Delta(S \times A)}
\sum_{\tau^i \prec \tau^j}
\log
\frac{\exp(r(\tau^j))}{\exp(r(\tau^i))+\exp(r(\tau^j))}.
\end{equation}
D-REX differs only in how the rankings are obtained: it automatically generates ranked trajectories by injecting noise into a cloned policy and then applies the same T-REX likelihood objective~\citep{brown2020better}.
In the one-step bandit counterexample below, a demonstration occupancy $d$ is scored by its expected one-step reward $r^\top d$; for a deterministic trajectory $\tau_a$ taking action $a$, this reduces to the action reward $r(\tau_a)=r(a)$.

\paragraph{Proof of Proposition~\ref{prop:baseline-failure} for the ranking-based estimator.}
Consider a one-state bandit with three actions $A=\{a_1,a_2,a_3\}$, so the occupancy set is
\begin{equation}
\Phi=\left\{d\in\mathbb R_+^3:\ \sum_{i=1}^3 d(a_i)=1\right\}.
\end{equation}
Let the ground-truth optimal occupancy be $d_e^\star=e_{a_1}$, induced for example by the reward $r^\star=e_{a_1}$.

Fix any $\eta>0$.
Choose $\rho>0$ small enough that
\begin{equation}
2(1-2\rho)\sigma(-(1-\rho))>\frac12,
\qquad
\sigma(x):=(1+e^{-x})^{-1}.
\end{equation}
Such a choice is possible because the left-hand side is continuous in $\rho$ and equals $2\sigma(-1)>1/2$ at $\rho=0$.
Consider the finite occupancy-suboptimality dataset
\begin{equation}
\mathcal D:=\{(d^g,\epsilon^g),(d^{u,1},\epsilon^u),(d^{u,2},\epsilon^u),(d^b,\epsilon^b)\},
\end{equation}
where
\begin{equation}
d^g=e_{a_1},\qquad \epsilon^g=0,
\end{equation}
\begin{equation}
d^{u,1}=d^{u,2}=\rho e_{a_1}+(1-\rho)e_{a_2},\qquad \epsilon^u=1-\rho,
\end{equation}
and
\begin{equation}
d^b=e_{a_3},\qquad \epsilon^b=1.
\end{equation}
For the ranking-based estimator, score each demonstration occupancy $d$ by $r^\top d$ and use the comparisons
\begin{equation}
d^{u,1}\succ d^b,\qquad d^{u,2}\succ d^b,\qquad d^g\succ d^b.
\end{equation}
These comparisons are consistent with $r^\star$, since $(r^\star)^\top d^g=1$, $(r^\star)^\top d^{u,i}=\rho$, and $(r^\star)^\top d^b=0$.

We first verify consistency with $\mathcal R^\star$.
If $r\in\mathcal R^\star$, then $a_1$ is optimal under $r$.
Hence
\begin{equation}
\subopt(r,d^g)=0=\epsilon^g,
\end{equation}
and for each $i\in\{1,2\}$,
\begin{equation}
\subopt(r,d^{u,i})=(1-\rho)\bigl(r(a_1)-r(a_2)\bigr)\le 1-\rho=\epsilon^u,
\end{equation}
while
\begin{equation}
\subopt(r,d^b)\le 1=\epsilon^b.
\end{equation}
Therefore $\mathcal R^\star\subseteq \mathcal R(\mathcal D)$.

Since $(d^g,\epsilon^g)\in\mathcal D\subseteq\conv(\mathcal D)$, Theorem~\ref{thm:bound1} yields
\begin{equation}
\min_{(d,\epsilon)\in\conv(\mathcal D)}\left\{\|d-d_e^\star\|_1+\epsilon\right\}
\le
\|d^g-d_e^\star\|_1+\epsilon^g
=
0,
\end{equation}
which is at most $\eta$.

We now analyze the actual T-REX likelihood.
The T-REX objective is
\begin{equation}
L(r)
=
2\log \sigma\!\bigl(\rho r(a_1)+(1-\rho)r(a_2)-r(a_3)\bigr)
+\log \sigma\!\bigl(r(a_1)-r(a_3)\bigr).
\end{equation}
Because $\log \sigma(\cdot)$ is increasing, any mass assigned to $a_3$ can be moved to $a_2$ while increasing both differences $\rho r(a_1)+(1-\rho)r(a_2)-r(a_3)$ and $r(a_1)-r(a_3)$.
Hence every maximizer satisfies
\begin{equation}
r(a_3)=0.
\end{equation}
Therefore we may write
\begin{equation}
r=(t,1-t,0),\qquad t\in[0,1],
\end{equation}
and the objective reduces to
\begin{equation}
L(t)=2\log \sigma\!\bigl(\rho t+(1-\rho)(1-t)\bigr)+\log \sigma(t).
\end{equation}
This function is concave, andD
\begin{equation}
L'(t)=-2(1-2\rho)\sigma\!\bigl(-\rho t-(1-\rho)(1-t)\bigr)+\sigma(-t).
\end{equation}
For every $t\in[0,1]$, we have $\sigma(-t)\le \tfrac12$ and $\sigma\!\bigl(-\rho t-(1-\rho)(1-t)\bigr)\ge \sigma(-(1-\rho))$, so
\begin{equation}
L'(t)\le \frac12-2(1-2\rho)\sigma(-(1-\rho))<0
\end{equation}
by the choice of $\rho$.
Thus $L$ is strictly decreasing on $[0,1]$, and every maximizer is
\begin{equation}
\hat r(a_1)=0,\qquad \hat r(a_2)=1,\qquad \hat r(a_3)=0.
\end{equation}
Finally,
\begin{equation}
\subopt(\hat r,d_e^\star)
=
\max_{d\in\Phi}\hat r^\top d-\hat r^\top d_e^\star
=
1-0
=
1.
\end{equation}
Thus, with $C=1$, we have constructed a finite dataset $\mathcal D$ for which the upper bound in Theorem~\ref{thm:bound1} is at most $\eta$, while every T-REX maximizer satisfies $\subopt(\hat r,d_e^\star)\ge C$.
Since D-REX uses the same reward-learning objective after generating automatic rankings, the same counterexample applies to that ranking-based baseline as well.

\section{Coverage-based penalty}\label{app:coverage_penalty}
In the single-optimal-demonstrator setting, reward ambiguity can be mitigated by selecting a (relative) interior point of the feasible reward set.
One standard construction imposes a positive margin on the Bellman inequalities, i.e., $(M^\top v -r)(s, a) \geq \delta  \mathbf{1}\{d_e(s,a) = 0\}$ for all $(s, a) \in S \times A$.
The margin $\delta$ penalizes reward on state-action pairs unvisited by the expert $d_e$ and discourages degenerate solutions on the boundary of the normal cone $N(d_e)$.

We extend this idea to multiple suboptimal demonstrations by penalizing state-action pairs that are unvisited by \emph{all} demonstrators simultaneously.
Concretely, we impose a positive margin on pairs that have zero mass under the demonstrators' aggregate occupancy measure.
The inactive region and the feasible reward set with coverage-based penalty are defined as follows.

\begin{definition}\label{def:inactive_region}
Let $\bar d$ denote the aggregate demonstrator occupancy measure: $\bar d := \frac{1}{K}\sum_{k=1}^K d_e^k$,
and $\mathcal{I}_\mathcal{D} := \{ (s,a) \in S \times A \mid \bar d(s,a) = 0 \}$ denote the inactive state-action pairs where the aggregate occupancy measure is zero.
\end{definition}

\begin{definition}[Feasible reward set with coverage-based penalty]\label{def:lp_multi_agg}
\begin{equation}\label{eq:lp_multi_agg}
    \begin{split}
    \mathcal{R}_{\delta}(\mathcal{D}) := &\{ r \in \Delta(S \times A)  \mid \; \exists  v \in \mathbb{R}^{|S|} \quad \text{s.t.} \quad (1-\gamma)\mu_0^\top v - r^\top d_e^k \leq \epsilon^k \quad \forall k \in [K],\\
    &(M^\top v - r)(s, a) \geq \delta \mathbf{1}\{ (s,a) \in \mathcal{I}_\mathcal{D} \} \quad \forall (s,a) \in S \times A\}.
    \end{split}
    \end{equation}
\end{definition}

By definition, $\mathcal{R}_0 (\mathcal{D}) =  \mathcal{R} (\mathcal{D})$ and $\mathcal{R}_\delta (\mathcal{D}) \subseteq  \mathcal{R} (\mathcal{D})$ for any $\delta>0$.
As in the single-optimal-demonstrator case, positive Bellman inequality slack on inactive pairs discourages rewards that keep clearly under-covered behaviors competitive.
The next theorem quantifies this effect: any occupancy placing mass off the demonstrators' aggregate support pays a suboptimality gap that scales linearly with the total off-support mass and the margin $\delta$.

\begin{theorem}\label{thm:sharpness}
If $r\in\mathcal R_{\delta}(\mathcal{D})$, then for any $d\in\Phi$, we have
\begin{equation}
    \subopt(r,d)\ge \delta \sum_{(s,a) \in \mathcal{I}_\mathcal{D}} d(s,a) - \min_{k\in[K]} \epsilon^k.
\end{equation}
\end{theorem}

\begin{proof}
Fix any $r\in\mathcal R_\delta(\mathcal D)$, and let $v\in\mathbb R^{|S|}$ be a feasible solution from \eqref{eq:lp_multi_agg}. Define the slack vector
\begin{equation}
l := M^\top v-r.
\end{equation}
Then $l(s,a)\ge \delta\,\mathbf 1\{(s,a)\in\mathcal I_\mathcal D\}$ for every $(s,a)$.

By weak duality, for every $d\in\Phi$ we have $r^\top d \le (1-\gamma)\mu_0^\top v$.
Therefore, for each demonstrator $k\in[K]$,
\begin{equation}
\subopt(r,d_e^k)
= \max_{d\in\Phi} \; r^\top d - r^\top d_e^k
\le (1-\gamma)\mu_0^\top v-r^\top d_e^k
\le \epsilon^k,
\end{equation}
where the last inequality is exactly the feasibility condition in~\eqref{eq:lp_multi_agg}. Hence $d_e^k$ is at most $\epsilon^k$-suboptimal for every $k\in[K]$.

Fix any $d\in\Phi$ and any demonstrator index $k\in[K]$. Since $d_e^k\in\Phi$, the definition of $\subopt$ gives
\begin{equation}
\subopt(r,d)
= \max_{\tilde d\in\Phi} r^\top \tilde d-r^\top d
\ge r^\top d_e^k-r^\top d.
\end{equation}
Using $r=M^\top v-l$ and the Bellman equation constraint $Md=(1-\gamma)\mu_0$ for any $d\in\Phi$, we obtain
\begin{equation}
r^\top d
= v^\top Md-l^\top d
= (1-\gamma)\mu_0^\top v-l^\top d.
\end{equation}
The same identity also holds with $d_e^k$ in place of $d$.
Hence
\begin{equation}
\subopt(r,d)
\ge l^\top d-l^\top d_e^k.
\end{equation}
Moreover, by the constraint of \eqref{eq:lp_multi_agg}, we have
\begin{equation}
l^\top d_e^k
= (M^\top v-r)^\top d_e^k
= (1-\gamma)\mu_0^\top v-r^\top d_e^k
\le \epsilon^k.
\end{equation}
Combining the last two inequalities yields
\begin{equation}
\subopt(r,d)
\ge l^\top d-\epsilon^k
\ge \delta \sum_{(s,a)\in\mathcal I_\mathcal D} d(s,a)-\epsilon^k.
\end{equation}
Since this holds for every $k\in[K]$, we conclude that
\begin{equation}
\subopt(r,d)\ge \delta \sum_{(s,a)\in\mathcal I_\mathcal D} d(s,a)-\min_{k\in[K]}\epsilon^k,
\end{equation}
which proves the claim.
\end{proof}

\section{Effect of additional pairwise performance-gap information}\label{app:pairwise_performance_gap}
Suppose we additionally observe \emph{pairwise performance-gap information} from comparisons: for some ordered pairs $(i,j)$, demonstrator $i$ is known to outperform demonstrator $j$ by at least a margin $\epsilon_{ij}\geq 0$ under the unknown ground-truth reward. In our occupancy-measure notation, this corresponds to the linear constraints
\begin{equation}\label{eq:pairwise}
    r^\top(d_e^i-d_e^j)\ \ge\ \epsilon_{ij}\qquad \forall (i,j)\in\mathcal P,
\end{equation}
where $\mathcal P\subseteq[K]\times[K]$ indexes the available comparisons.

This section shows that the constraints in \eqref{eq:pairwise} rule out degenerate rewards by forcing any recovered reward to induce a nontrivial value range between the best and worst policies.
To be specific, we consider the following worst-case reward range: among all rewards in $\mathcal R$, it is the smallest best-vs-worst reward range achievable over $\Phi$.
\begin{definition}[Worst-case reward range under $\mathcal{R}$]\label{def:reward_range}
\begin{equation}
\mathrm{rng}(\mathcal{R}) := \min_{r \in \mathcal{R}} \left( \max_{\tilde d \in \Phi} r^\top \tilde d - \min_{d \in \Phi} r^\top d \right).
\end{equation}
\end{definition}

Then, we show that this measure is larger than the sum of performance gaps over any path in the comparison graph.

\paragraph{Aggregating comparisons via paths.}
When many comparisons are available, one can aggregate them along a chain of pairwise comparisons to obtain a stronger lower bound than $\max_{(i,j)\in\mathcal P}\epsilon_{ij}$.
Define the directed \emph{comparison graph} $\mathcal G=(\mathcal V,\mathcal E)$ with $\mathcal V:=[K]$ and $\mathcal E:=\mathcal P$, and assign each edge weight $\epsilon_{ij}$ between demonstrators $i$ and $j$ by the pairwise performance-gap information.
For a simple directed path
\begin{equation}
    p=(i_0\to i_1\to \cdots \to i_m),
\end{equation}
in $\mathcal G$, define its total weight as
\begin{equation}
    \epsilon(p):=\sum_{\ell=1}^m \epsilon_{i_{\ell-1}i_\ell}.
\end{equation}
The next result shows that every such path yields a lower bound on the worst-case reward range.

\begin{proposition}[Path-based worst-case reward-range guarantee]\label{prop:range_path}
Let
\begin{equation}
    \mathcal R_{\mathcal P}
    :=
    \left\{
    r\in\mathcal R(\mathcal D)
    \;\middle|\;
    r^\top(d_e^i-d_e^j)\ge \epsilon_{ij}
    \quad \forall (i,j)\in\mathcal P
    \right\}.
\end{equation}
For any simple directed path $p$ in $\mathcal G$, if $\mathcal R_{\mathcal P}$ is nonempty, then
\begin{equation}\label{eq:range_path}
    \mathrm{rng}(\mathcal R_{\mathcal P})\ \ge \epsilon(p).
\end{equation}
\end{proposition}
\begin{proof}
Fix any simple directed path $p$ in $\mathcal G$ of the form $p=(i_0 \to i_1 \to \cdots \to i_m)$, where each $(i_{\ell-1},i_\ell)\in\mathcal P$.
Fix any $r\in\mathcal R_{\mathcal P}$.
By the definition of $\mathcal R_{\mathcal P}$, we have for each edge
$r^\top(d_e^{i_{\ell-1}}-d_e^{i_\ell})\ge \epsilon_{i_{\ell-1}i_\ell}$.
Summing over $\ell=1,\ldots,m$ telescopes to
\begin{equation}
    r^\top(d_e^{i_0}-d_e^{i_m}) \;\ge\; \epsilon(p).
\end{equation}
Since $d_e^{i_0},d_e^{i_m}\in\Phi$, the reward range under $r$ satisfies
\begin{equation}
    \max_{\tilde d\in\Phi} r^\top \tilde d - \min_{d\in\Phi} r^\top d
    \;\ge\;
    r^\top d_e^{i_0} - r^\top d_e^{i_m}
    \;\ge\;
    \epsilon(p).
\end{equation}
Since $r\in\mathcal R_{\mathcal P}$ was arbitrary, taking the minimum over $r\in\mathcal R_{\mathcal P}$ yields
\begin{equation}
\mathrm{rng}(\mathcal R_{\mathcal P})
\;=\;
\min_{r\in\mathcal R_{\mathcal P}}
\left[\max_{\tilde d\in\Phi} r^\top \tilde d - \min_{d\in\Phi} r^\top d \right]
\;\ge\;
\epsilon(p).
\end{equation}
\end{proof}

\section{Advantage parameterization of the offline objective}\label{app:direct_policy}
We derive the advantage-parameterized loss from the population optimal advantage, and then use the resulting expression with empirical occupancies.

\begin{proposition}[Advantage parameterization]\label{prop:advantage_param}
For a reward $r$, let $v_r^\star$ be its optimal value function and define the optimal advantage
\begin{equation}\label{eq:optimal_adv_appendix}
    A_r^\star(s,a)
    :=
    r(s,a)
    + \gamma \sum_{s'\in S}P(s'\mid s,a)v_r^\star(s')
    - v_r^\star(s).
\end{equation}
Then $A_r^\star(s,a)\le 0$ for every $(s,a)$, and for every occupancy $d\in\Phi$,
\begin{equation}\label{eq:subopt_adv_identity}
    \subopt(r,d)
    =
    -\sum_{s,a} d(s,a)A_r^\star(s,a).
\end{equation}
Consequently, the population version of the constrained reward-selection objective can be written directly in terms of $A_r^\star$.
Replacing the population quantities by empirical occupancies and using a free parameterization $A_\psi$ gives
\begin{align}\label{eq:loss_advantage_param}
    L(A_\psi) :=\;&
    -\sum_{s,a}\left(\frac{1}{K}\sum_{k=1}^K \hat d_e^k(s,a)-\hat d_{\mathrm{base}}(s,a)\right)A_\psi(s,a)
    \nonumber\\
    &+ \lambda_{\mathrm{sub}}\frac{1}{K}\sum_{k=1}^K
    \rho_{\leq}\!\left(
    -\sum_{s,a}\hat d_e^k(s,a)\,A_\psi(s,a)-\epsilon^k
    \right)
    \nonumber\\
    &+ \lambda_{\mathrm{bell}}\,
    \mathbb{E}_{(s,a)\sim\bar{\mathcal{D}}}
    \left[
    \rho_{\leq}\!\left(A_\psi(s,a)\right)
    \right].
\end{align}
\end{proposition}

\begin{proof}
The Bellman optimality equation gives $A_r^\star(s,a)\le 0$ for all $(s,a)$.
Since $d\in\Phi$ satisfies $Md=(1-\gamma)\mu_0$, we have
\begin{align}
    \sum_{s,a}d(s,a)A_r^\star(s,a)
    &=
    r^\top d - (M^\top v_r^\star)^\top d \\
    &=
    r^\top d - (1-\gamma)\mu_0^\top v_r^\star \\
    &=
    r^\top d - \max_{\tilde d\in\Phi} r^\top \tilde d
    =
    -\subopt(r,d),
\end{align}
which proves \eqref{eq:subopt_adv_identity}.
Thus the population suboptimality constraint $\subopt(r,d_e^k)\le\epsilon^k$ is equivalent to
\begin{equation}
    -\sum_{s,a}d_e^k(s,a)A_r^\star(s,a)-\epsilon^k\le 0.
\end{equation}
Similarly, for any two occupancies $d,d'\in\Phi$, the flow constraints imply
\begin{equation}
    r^\top(d-d')
    =
    \sum_{s,a}\bigl(d(s,a)-d'(s,a)\bigr)A_r^\star(s,a),
\end{equation}
because $(M^\top v_r^\star)^\top(d-d')=v_r^{\star\top}M(d-d')=0$.
Applying this identity with $d=\frac{1}{K}\sum_{k=1}^K d_e^k$ and $d'=d_{\mathrm{base}}$ rewrites the reward-selection term.
Finally, Bellman feasibility is exactly the sign constraint $A_r^\star(s,a)\le0$, which is penalized by $\rho_{\leq}(A_\psi(s,a))$ in the direct parameterization.
Replacing $d_e^k$ and $d_{\mathrm{base}}$ by their empirical estimates and sampling the sign penalty over $\bar{\mathcal D}$ yields \eqref{eq:loss_advantage_param}.
\end{proof}
In implementation, one may further restrict the direct parameterization to non-positive advantages, for example by setting $A_\psi(s,a)=-\mathrm{softplus}(f_\psi(s,a))$.
This enforces the Bellman-feasibility sign convention by construction while preserving the same penalty-based objective form.

\section{Experimental details}\label{app:exp_details}
\subsection{Grid-world with detour portals}
We illustrate the two cases of suboptimal demonstrators in Figure~\ref{fig:grid_world_illustration}.
In each figure, the yellow star marks the terminal goal state, and the red, blue, and green arrows indicate the active portal transitions from entrance to exit.
The black arrows show the optimal policy.
In Case 1, the demonstrator quality, or knowledge, increases with $K$; in Case 2, only the state-action coverage increases with $K$ while the demonstrator quality remains limited.

\begin{figure}[H]
    \centering
    \begin{subfigure}[t]{0.32\textwidth}
        \centering
        \includegraphics[width=\textwidth]{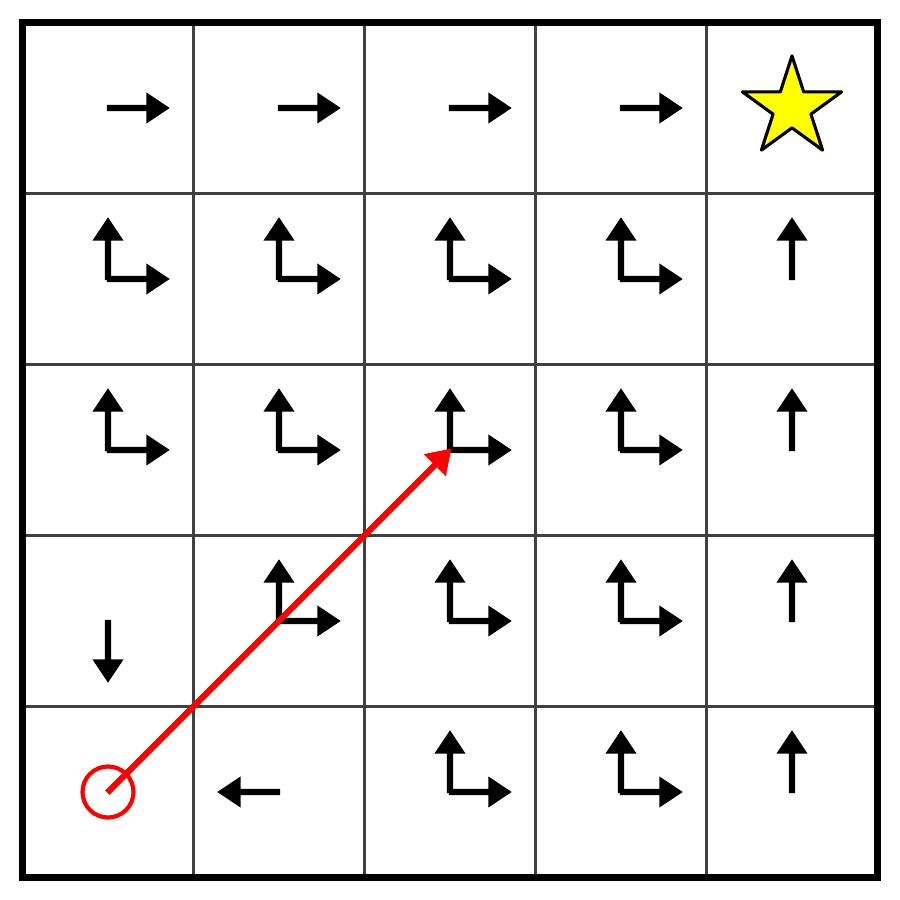}
        \caption{Case 1: Demonstrator 1}
    \end{subfigure}
    \hfill
    \begin{subfigure}[t]{0.32\textwidth}
        \centering
        \includegraphics[width=\textwidth]{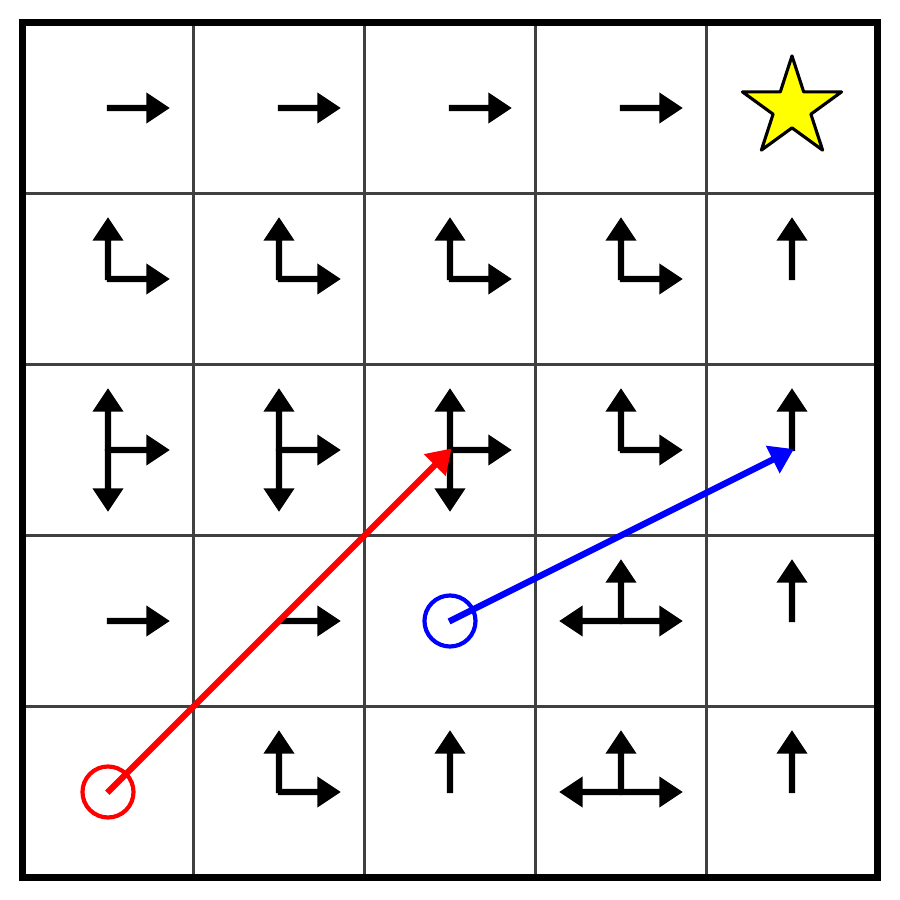}
        \caption{Case 1: Demonstrator 2}
    \end{subfigure}
    \hfill
    \begin{subfigure}[t]{0.32\textwidth}
        \centering
        \includegraphics[width=\textwidth]{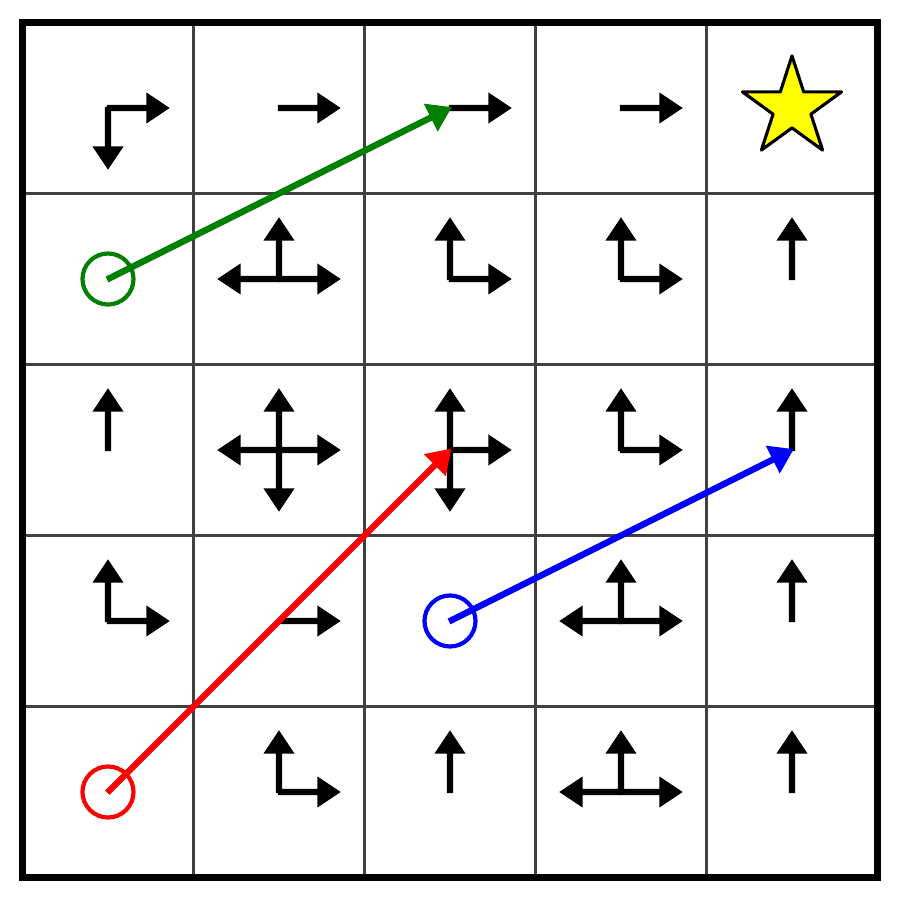}
        \caption{Case 1: Demonstrator 3}
    \end{subfigure}

    \vspace{0.75em}

    \begin{subfigure}[t]{0.32\textwidth}
        \centering
        \includegraphics[width=\textwidth]{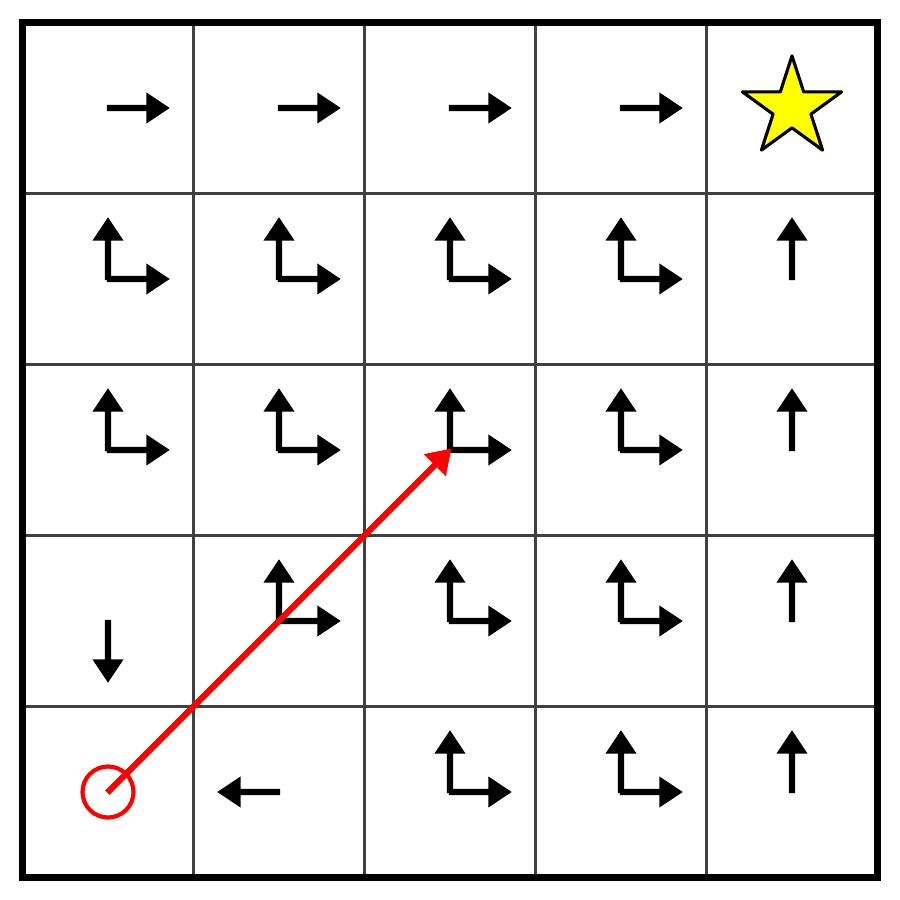}
        \caption{Case 2: Demonstrator 1}
    \end{subfigure}
    \hfill
    \begin{subfigure}[t]{0.32\textwidth}
        \centering
        \includegraphics[width=\textwidth]{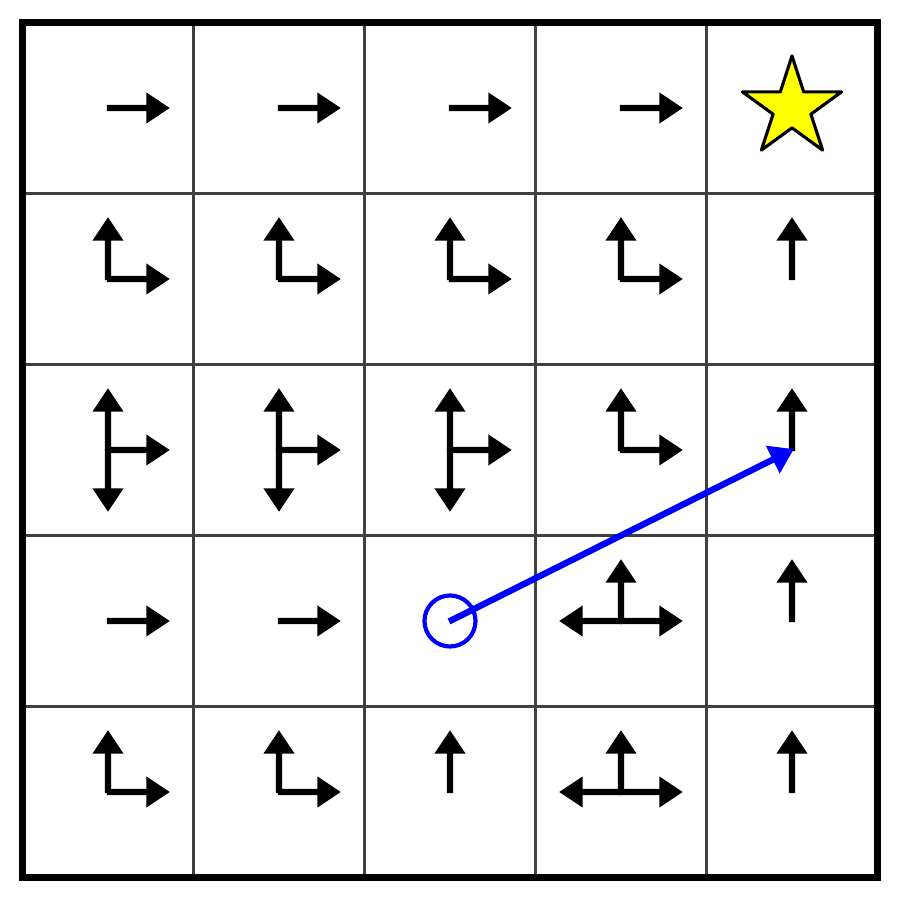}
        \caption{Case 2: Demonstrator 2}
    \end{subfigure}
    \hfill
    \begin{subfigure}[t]{0.32\textwidth}
        \centering
        \includegraphics[width=\textwidth]{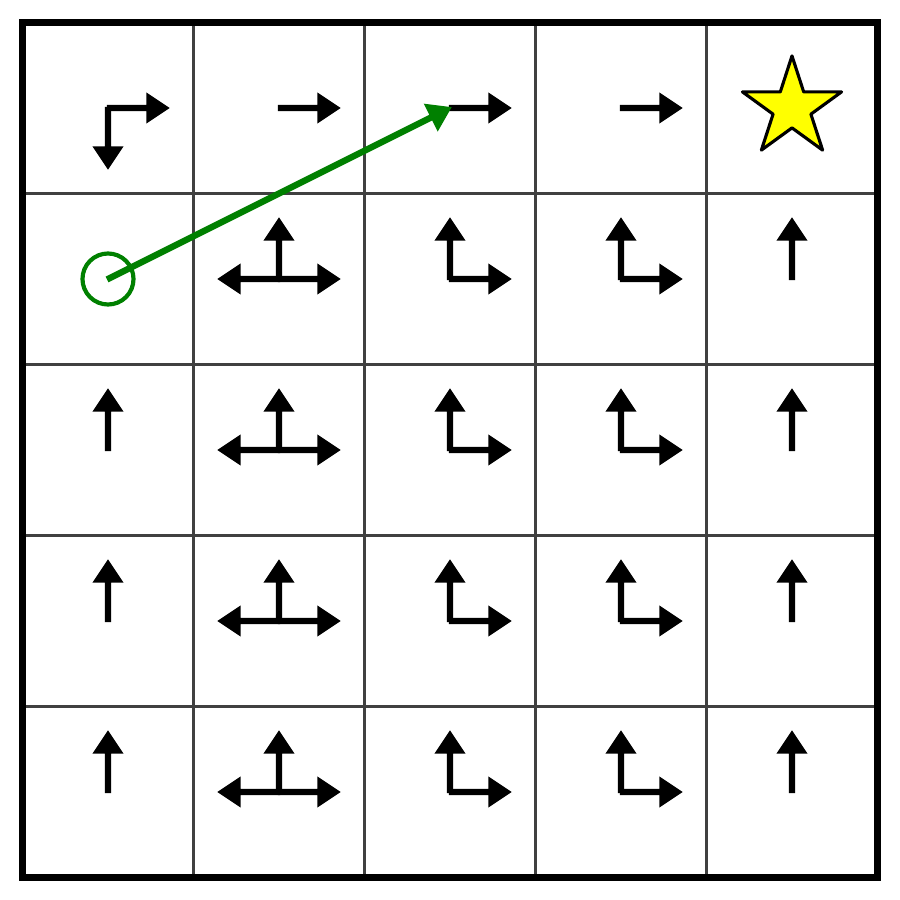}
        \caption{Case 2: Demonstrator 3}
    \end{subfigure}
    \caption{Grid-world examples with detour portals for the two cases of suboptimal demonstrators.}
    \label{fig:grid_world_illustration}
\end{figure}

\subsection{Arithmetic tool-use LLM experiment}
This appendix provides additional details for the arithmetic tool-use LLM experiment in Section~\ref{sec:4}.

\paragraph{Task and dataset generation.}
We consider synthetic arithmetic tool-use tasks built from five tools:
\texttt{add}, \texttt{subtract}, \texttt{multiply}, \texttt{divide}, and \texttt{power}.
Each task consists of a natural-language arithmetic question whose solution requires a chain of one to five tool calls.
The target completion is a multi-line tool-use trace with one call per line.
When an argument is the output of the previous line, we replace it with the special token \texttt{PREV}.

For each training task, we construct one optimal completion and five demonstrator completions.
The optimal completion uses the correct tool at every step.
Demonstrator $k$ is specialized to one tool: whenever the correct tool on a line matches that specialty, demonstrator $k$ produces the correct line; otherwise, it outputs a fixed placeholder tool ``unknown'' with the correct arguments.
This construction yields five suboptimal demonstrators with complementary coverage over the tool set.
For the proposed method, we use a fixed baseline completion that replaces every tool call in the optimal trace with the same placeholder tool ``unknown''.

\paragraph{Model and training details.}
We initialize from \texttt{Qwen/Qwen2.5-1.5B}~\cite{yang2025qwen} and attach a single linear advantage head on top of the final hidden states.
The base checkpoint is available at \url{https://huggingface.co/Qwen/Qwen2.5-1.5B} under the Apache License 2.0.
The backbone is fine-tuned with Low-Rank Adaptation (LoRA).

For the proposed method, we use the advantage-parameterized objective \eqref{eq:loss_advantage_param} in Appendix~\ref{app:direct_policy}.
We use the squared hinge loss for $\rho_{\leq}$.
As a baseline, we compare against a MaxEnt objective that fits the demonstrator completions with demonstrator-specific inverse temperatures, as described in Appendix~\ref{app:maxent}.
We optimize all trainable parameters with the AdamW optimizer.
All LLM fine-tuning experiments were run on a single NVIDIA L40S GPU.
Using the same model, dataset size, batch size, and number of epochs, the proposed advantage-parameterized objective has nearly the same wall-clock training time as the MaxEnt IRL baseline, since both optimize the same backbone with comparable loss computations.
Additional optimization hyperparameters are reported in Table~\ref{tab:llm_hparams}.
\begin{table}[t]
    \centering
    \footnotesize
    \caption{LLM experiment hyperparameters.}
    \label{tab:llm_hparams}
    \begin{tabular}{ll}
        \toprule
        Field & Value \\
        \midrule
        Train / val tasks & \texttt{2000} / \texttt{200} \\
        Epochs / batch size / grad.\ accum. & 1 / 8 / 4 \\
        Learning rate & $10^{-4}$ \\
        Optimizer & AdamW \\
        AdamW $\beta_1$ / $\beta_2$ / $\varepsilon$ & 0.9 / 0.999 / $10^{-8}$ \\
        Weight decay / warmup ratio / grad.\ clip & 0.01 / 0.1 / 1.0 \\
        LoRA $r$ / $\alpha$ / dropout & 16 / 32 / 0.05 \\
        $\gamma$ (discount factor)  & 0.95 \\
        $\epsilon^k$ (suboptimality slack) & 0.1 \\
        $\lambda_{\mathrm{sub}}$ / $\lambda_{\mathrm{bell}}$ & 1.0 / 1.0 \\
        Penalty function $\rho_{\leq}(x)$ & $\max\{0,x\}^2$ (squared hinge) \\
        $\beta$ (MaxEnt inverse temperature) & 1.0 \\
        \bottomrule
    \end{tabular}
\end{table}

\paragraph{Evaluation.}
We evaluate the trained model on the validation set by scoring the optimal completion and a near-optimal completion with only a single wrong tool name, sampled uniformly at random from the available tools.
For validation task $i$, let $y_i^\star$ denote the optimal completion, $\tilde y_i$ denote the near-optimal completion, and $S_\theta(x_i,y)$ denote the score assigned by the trained reward or advantage model to completion $y$ for prompt $x_i$.
We report the win-rate
\begin{equation}
    \mathrm{WinRate}
    :=
    \frac{1}{N_{\mathrm{val}}}
    \sum_{i=1}^{N_{\mathrm{val}}}
    \mathbf{1}\!\left\{S_\theta(x_i,y_i^\star) > S_\theta(x_i,\tilde y_i)\right\},
\end{equation}
with mean and standard deviation computed over 10 runs.

\section{Societal impact}\label{app:societal_impact}

This work is primarily a theoretical contribution to IRL from multiple imperfect demonstrators.
Potential positive impacts include reducing reliance on a single near-optimal demonstrator and enabling reward learning from more realistic heterogeneous data, which may be useful in applications such as RLHF, robotics, and assistive decision-making where demonstrations from optimal demonstrators are costly or unavailable.
At the same time, reward-learning methods can inherit biases or systematic errors present in the demonstrations, so care is needed when applying them in safety-critical or high-stakes domains.



\end{document}